\documentclass{article}
\usepackage[preprint]{neurips_2025}
\usepackage[utf8]{inputenc}
\usepackage[T1]{fontenc}
\usepackage[colorlinks=true, citecolor=blue, linkcolor=red]{hyperref}
\usepackage{url}
\usepackage{booktabs}
\usepackage{amsfonts}
\usepackage{nicefrac}
\usepackage{microtype}
\usepackage{xcolor}
\usepackage{amsmath}
\usepackage{svg}
\usepackage{amssymb}
\usepackage{amsthm}
\usepackage{thmtools}
\usepackage{bm}
\usepackage{tikz}
\usepackage[scr]{rsfso}
\usepackage{float}

\newtheorem{theorem}{Theorem}
\newtheorem{lemma}{Lemma}
\newtheorem{definition}{Definition}

\newtheorem*{mytheorem}{Theorem}
\newtheorem*{theorem*}{Theorem}

\newtheorem{corollary}{Corollary}

\title{Quantization vs Pruning: Insights from the Strong Lottery Ticket Hypothesis}

\author{%
    Aakash Kumar \\
    Department of Physical Sciences, IISER Kolkata, \\
    West Bengal, India 741246 \\
    \texttt{ak20ms209@iiserkol.ac.in} \\
    \And
    Emanuele Natale \\
    Université Coté d’Azur,\\
    CNRS, Inria, I3S, France\\
    \texttt{emanuele.natale@univ-cotedazur.fr} \\
}

\begin{document}
\maketitle

\begin{abstract} 
    Quantization is an essential technique for making neural networks more efficient, yet our theoretical understanding of it remains limited. Previous works demonstrated that extremely low-precision networks, such as binary networks, can be constructed by pruning large, randomly-initialized networks, and showed that the ratio between the size of the original and the pruned networks is at most polylogarithmic.

The specific pruning method they employed inspired a line of theoretical work known as the Strong Lottery Ticket Hypothesis (SLTH), which leverages insights from the Random Subset Sum Problem. However, these results primarily address the continuous setting and cannot be applied to extend SLTH results to the quantized setting.

In this work, we build on foundational results by Borgs et al. on the Number Partitioning Problem to derive new theoretical results for the Random Subset Sum Problem in a quantized setting. 
Using these results, we then extend the SLTH framework to finite-precision networks. While prior work on SLTH showed that pruning allows approximation of a certain class of neural networks, we demonstrate that, in the quantized setting, the analogous class of target discrete neural networks can be represented exactly, and we prove optimal bounds on the necessary overparameterization of the initial network as a function of the precision of the target network.
\end{abstract}

\section{Introduction}
    \label{sec:int}
    Deep neural networks (DNNs) have become ubiquitous in modern machine‐learning systems, yet their ever‑growing size quickly collides with the energy, memory, and latency constraints of real‑world hardware. \textbf{Quantization}—representing weights with a small number of bits—is arguably the most hardware‑friendly compression technique, and recent empirical work shows that aggressive quantization can preserve accuracy even down to the few bits regime. Unfortunately, our theoretical understanding of why and when such extreme precision reduction is possible still lags far behind in practice.
An interesting step in this direction was the \textit{Multi‑prize Lottery Ticket Hypothesis} (MPLTH) put forward by \cite{multiprize}. They \textit{empirically} demonstrated that a sufficiently large, randomly initialized network contains sparse \textit{binary} subnetworks that match the performance of a target network with real‑valued weights. 
They also provided theoretical guarantees regarding the existence of such highly quantized networks, showing that, with respect to the target network, the initial random network need only be larger by a \textbf{polynomial factor}. \cite{Sreenivasan} subsequently improved this bound, by showing that a polylogarithmic factor is sufficient (See Section \ref{sec:rel}). 
These works fall within the research topic known as the \textit{Strong Lottery Ticket Hypothesis} (SLTH), which states that sufficiently-large randomly initialized neural networks contain subnetworks, called \emph{lottery tickets}, that perform well on a given task, without requiring weight adjustments.
The main theoretical question, therefore, is: how large should the initial network be to ensure it contains a lottery ticket capable of approximating a given family of target neural networks?
Research on the SLTH, however, has mainly focused on investigating pruning in the {continuous-weight} (i.e., not quantized) setting, drawing on results for the Random Subset Sum Problem (RSSP)  \citep{Lueker1998} to show that over‑parameterized networks can be pruned to \emph{approximate} any target network without further training \citep{orseauLogarithmicPruningAll2020,pensia,Burkholz} (for additional context on this body of literature, we kindly refer the reader to the Related Work in Section \ref{sec:rel}).
However, the analytic RSSP results used for SLTH rely heavily on real‑valued weights and therefore do not extend to the finite‑precision regime considered in the MPLTH. This gap left open a fundamental question:
\begin{quote}
     What is the over‑parameterization needed to obtain \emph{quantized} strong lottery tickets?
\end{quote}

\paragraph{Our contributions.}
We address the aforementioned gap by revisiting the classic \textit{Number Partitioning Problem} (NPP), which is closely related to the RSSP. Building on the seminal results of \cite{Borgs2001} concerning the \textit{phase transition} of NPP, we derive new, sharp bounds for the \textbf{discrete RSSP}. These bounds are precise enough to adapt the SLTH proof strategy to the finite precision setting and, in doing so, establish optimal bounds for the MPLTH. Crucially, our results account for arbitrary quantization in both the initial and target networks, and demonstrate that the lottery ticket can represent the target network \emph{exactly}. In contrast, prior work limited the initial network to binary weights and assumed continuous weights for the larger (target) network \citep{multiprize, Sreenivasan}, requiring a cubic overparameterization relative to the lower bound and additional dependencies on network parameters absent in our bound.
Concretely, let $\delta_t$ denote the precision (i.e., quantization level) of a target network $N_t$, $\delta_{\text{in}}$ the precision of a randomly initialized larger network $N_{\text{in}}$, and $\delta$ any parameter satisfying $\delta_t \geq \delta^2 \geq \delta_{\text{in}}^2$. Denote by $d$ and $\ell$ the width and depth of the target network, respectively. Our results can be summarized by the following simplified, informal theorem (refer to the formal statements for full generality).

\begin{theorem*}[Informal version of Theorems \ref{thm:main}, \ref{thm:main_2}, and \ref{thm:lower}]
    With high probability\footnote{As customary in the literature on randomized algorithms, with the expression \emph{with high probability} we refer to a probability of failure which scales as the inverse of a polynomial in the parameter of interest (the number of precision bits $\log 1/\delta)$ in our case).}, a depth‑$2\ell$ network $N_{\text{in}}$ of width $\mathcal{O}\bigl(d \log(1/\delta)\bigr)$ can be quantized to precision $\delta$ and pruned to become functionally equivalent to any $\delta_t$-quantized target network $N_t$ with layers of width at most $d$ (Theorem~\ref{thm:main}). This result is optimal, as no two‑layer network of precision $\delta$ with fewer than $\Omega\bigl(d \log(1/\delta)\bigr)$ parameters can be pruned to represent $\delta$-quantized neural networks of width $d$ (Theorem~\ref{thm:lower}). Furthermore, the depth of $N_{\text{in}}$ can be reduced to $(\ell + 1)$ at the cost of an additional $\log(1/\delta)$ factor in its width (Theorem~\ref{thm:main_2}).
\end{theorem*}

These are the first theoretical results that (i) characterize the precise interplay between weight precision and over‑parameterization, and (ii) certify that pruning can yield \emph{exact}, not just approximate, quantized subnetworks. Besides contributing to the theory of network compression, our analysis showcases the versatility of classical combinatorial insights—such as the theory of NPP—in deep‑learning theory.

\paragraph{Paper organization.} 
In Section~\ref{sec:rel}, we review prior work on SLTH and quantization.
In Section~\ref{sec:QRSS}, we prove a new quantized version of the RSSP, after first recalling classical results on RNPP in subsection \ref{ssec:RNPP}. 
Our new theorems on the quantized SLTH are proved in Sections~\ref{sec:slth}, after recalling necessary notation and definitions in subsection \ref{ssec:preliminaries}. 
Finally, in Section~\ref{sec:con}, we discuss the conclusion of our work and future directions.

\section{Related Work}
    \label{sec:rel}
    \paragraph{Strong Lottery Ticket Hypothesis.}
In 2018, \cite{LTH} proposed the Lottery Ticket Hypothesis, which states that every dense network contains a sparse subnetwork that can be trained from scratch, and performs equally well as the dense network. 
Rather surprisingly, \cite{SLTH_main},  \cite{ramanujan} and \cite{pruningscratch} empirically showed that it is possible to efficiently find subnetworks within large randomly initialized networks that perform well on a given task, without changing the initial weights. 
This motivated the Strong Lottery Ticket Hypothesis (SLTH), which states that sufficiently overparameterized randomly initialized neural networks contain sparse subnetworks that will perform as well as a small trained network on a given dataset, without any training. 
Many formal results rigorously proved the SLTH in various settings, the first one being \cite{slthproof}, where they showed that a feed-forward dense target network of width $\ell$ and depth $d$ can be approximated by pruning a random network of depth $2\ell$ and width $\mathcal{O}(d^5\ell^2)$. 
\cite{orseauLogarithmicPruningAll2020, pensia} improved this bound by proving that width $\mathcal{O}(d\log (d\ell))$ is sufficient. Another construction was provided by \cite{Burkholz}, where they showed that a network of width $\ell+1$ is enough to approximate a network of width $\ell$, with a certain compromise on the width. 
Other works extended the SLTH to other famous architectures, such as convolutional \cite{burkholzConvolutionalResidualNetworks2022} and equivariant networks \cite{ferbachGeneralFrameworkProving2022}.
Next, we provide an informal version that qualitatively summarizes this kind of results.
\begin{mytheorem}[Informal qualitative template of SLTH results]
    With high probability, a random artificial neural network $N_R$ with $m$ parameters can be pruned so that the resulting subnetwork approximates, up to an error $\epsilon$, any target artificial neural network $N_T$ with $\mathcal{O}(m/\log_2(1/\epsilon))$ parameters. 
    The logarithmic dependency on $\epsilon$ is optimal. 
\end{mytheorem}

\paragraph{Quantization.} Neural network quantization refers to the process of reducing the precision of the weights within a neural network. Empirical studies have demonstrated that trained neural networks can often be significantly quantized without incurring substantial loss in performance~\cite{Han}. In particular,~\cite{multiprize} provided both empirical and theoretical support for a quantized variant of the SLTH, introducing an algorithm capable of training binary networks effectively. With regard to theoretical guarantees, they proved that a neural network with width $d$ and depth $\ell$ can be approximated to within an error $\epsilon$ by a binary target network of width $\mathcal{O}(\ell d^{3/2}/\epsilon + \ell d \log(\ell d/\epsilon))$. Subsequently,~\cite{Sreenivasan} presented an exponential improvement over this result, demonstrating that a binary network with depth $\Theta(\ell \log(d\ell/\epsilon))$ and width $\Theta(d \log^2(d\ell/\epsilon))$ suffices to approximate any given network of width $d$ and depth $\ell$. We remark that both of these results assume that the initial network weights are binary, whereas the target network weights are continuous.
The success of techniques to construct heavily quantized networks can be related to theoretical work that show that heavily quantized networks still retain good universal approximation properties \citep{hwang2024expressive}. 
In practice, not all parts of a network need to be quantized equally aggressively. Mixed-precision quantization allocates different bit-widths to different layers or parameters to balance accuracy and efficiency \citep{AutomaticMixedPrecision,GentleIntroduction8bit}.

\paragraph{Subset Sum Problem (SSP).} Given a target value $z$ and a multiset $\Omega$ of $n$ integers from the set $\{-M, -M+1,\dots,M-1,M\}$, the SSP consists in finding a subset $S\subseteq \Omega$ such that the sum of its elements equal $z$. In Computational Complexity Theory, SSP is one of the most famous NP-complete problems \cite{Garey1979-ol}. 
Its random version, Random SSP (RSSP), has been investigated since the 80s in the context of combinatorial optimization \citep{luekerAverageDifferenceSolutions1982,Lueker1998}, and recently received renowned attention in the machine learning community because of its connection to the SLTH \citep{pensia}.

\paragraph{Number Partitioning Problem (NPP).} NPP is the problem of partitioning a multiset $\Omega$ of $n$ integers from the set $[M]:=\{1,2,\dots,M\}$ into two subsets such that the difference of their respective sums equals a target value $z$ (typically, the literature has focused on minimizing $|z|$, i.e. trying to approximate the value closest to zero). Analogously to the aforementioned SSP, NPP is one of the most important NP-complete problems \citep{Garey1979-ol,Hayes2002}. Its random version, in which the $n$ elements are sampled uniformly at random from $[M]$, has also received considerable attention in Statistical Physics, where it has been shown to exhibit a phase transition \cite{Mzard2009}. Concretely, \cite{Mertens1998} heuristically showed the following result, which was later put on rigorous grounds by \cite{Borgs2001}: 
defining $\kappa := \frac{\log_2M}{n}$, if $\kappa<1$ then $\mathcal{O}(2^n)$ number of solutions exist, whereas if $\kappa>1$ the number of solutions sharply drops to zero.

\section{Quantized Random Subset Sum Problem}
    \label{sec:QRSS}
    \subsection{Random Number Partitioning Problem}
\label{ssec:RNPP}
In this section, we recall seminal results by \cite{Borgs2001} which we leverage in Subsection \ref{ssec:rssp} to obtain new results on the quantized RSSP.

\begin{definition}[RNPP]
    Let $\bm{X}=(X_1, X_2,\dots,X_n)$ be a set of integers sampled uniformly from the set $\{1,2,\dots,M\}$. The Random Number Partitioning Problem is defined as the problem of finding a partitioning set $\bm{\sigma} = (\sigma_1,\sigma_2,\dots,\sigma_n)$ with $\sigma_j \in \{-1,1\}$ such that $\bm{|\sigma\cdot x|} = z$ for some given integer $z$ (called target).
\end{definition}
Note that usually, in RNPP the difference between the sum of two parts is minimized, but we consider RNPP with a target, i.e., the difference between the sum of two parts must be equal to a given number $z$, the target. Given an instance of Random Number Partitioning Problem $\bm{X}=(X_1, X_2,\dots,X_n)$ with a set of size $n$ and a target $z$, $Z_{n,z}$ denotes the number of exact solutions to the RNPP, i.e.,
    \begin{equation*}
        Z_{n,z} = \sum_{\bm{\sigma}}\mathbb{I} (|\bm{\sigma\cdot X}|=z).
    \end{equation*}
To prove the existence of phase transition, \cite{Borgs2001} estimated the moments of $Z_{n,z}$. The relevant result is stated as Theorem \ref{thm:Borgs_main} (Appendix \ref{chap:npp_rssp}). Using these moment estimates of $Z_{n,z}$, one can write an upper and a lower bound on the probabilities of existence of solutions to a RNPP. We do so in Lemma \ref{Lemma:Upper_Lower_NPP}.
\begin{lemma}
\label{Lemma:Upper_Lower_NPP}
Given a Random Number Partitioning Problem, the probability $\mathbb{P}(Z_{n,z}>0)$ is bounded above and below as
\begin{equation*}
    \mathbb{P}(Z_{n,z}>0)\leq \begin{cases}
        \frac{\rho_n}{2}\left(\exp\left(-\frac{z^2}{2nM^2c_M}\right)+\mathcal{O}\left(\frac{1}{n}\right)\right)\;\;\; if \;\;z=0\\
        \rho_n\left(\exp\left(-\frac{z^2}{2nM^2c_M}\right)+\mathcal{O}\left(\frac{1}{n}\right)\right)\;\;\; if \;\;z \neq0
    \end{cases}
\end{equation*}
\begin{equation*}
    \mathbb{P}(Z_{n,z}>0)\geq \frac{1}{2\left(1+\exp\left(\frac{z^2}{nM^2c_M}\right)\left(\mathcal{O}\left(\frac{1}{n\rho_n}\right)+\mathcal{O}\left(\frac{1}{n}\right)\right)+\frac{1}{\rho_n}\right)}
\end{equation*}
where $\rho_n = 2^{n+1}\gamma_n$.
\end{lemma}
\begin{proof}[Proof Sketch]
    We use Markov's Inequality (Theorem \ref{ineq:markov} in Appendix \ref{chap:ineq}) and Cauchy-Schwartz inequality (Theorem \ref{ineq:CS} in Appendix \ref{chap:ineq}) to get bounds on the probabilities from the moment estimates (Theorem \ref{thm:Borgs_main}). See Appendix \ref{chap:npp_rssp} for details.
\end{proof}
The existence of phase transition (Section \ref{sec:rel}) is a consequence of Lemma \ref{Lemma:Upper_Lower_NPP} but for the purposes of this paper, we only require Lemma \ref{Lemma:Upper_Lower_NPP}.

\subsection{Quantized Random Subset Sum Problem}
\label{ssec:rssp}

The Random Subset Sum Problem (RSSP) is the problem of finding a subset of a given set such that the sum of this subset equals a given target $t$. RSSP is a crucial tool in proving results on SLTH \cite{pensia} \cite{Burkholz}. RSSP and RNPP are closely related, and hence we can use the results on RNPP in this section to make statements about RSSP. We shall then use these results on RSSP to prove results on SLTH and quantization.
\begin{definition}[RSSP]
    Let $\bm{X}=(X_1, X_2,\dots,X_n)$ be a set of integers sampled uniformly from the set $\{-M,\dots,1,2,3,\dots,M\}$. The RSSP is defined as the problem of finding an index set $S\subset [n]$ such that $\sum_{i\in S}X_i = t$ for a given integer $t$, called the target. 
\end{definition}
\begin{lemma}
\label{lemma:npp_is_rssp}
    An SSP with given set $\mathbf{X} = (X_1, X_2,\dots,X_n)$ and target $t$ can be solved iff the NPP can be solved on the given set $\mathbf{X}$ and target $\Lambda-2t$ (or $2t-\Lambda$), where $\Lambda=\sum_{i=1}^n X_i$.
\end{lemma}
Lemma \ref{lemma:npp_is_rssp} is proved in Appendix \ref{chap:npp_rssp}. Using the equivalence of RNPP and RSSP, the following results on RSSP follows from Lemma \ref{Lemma:Upper_Lower_NPP}.
\begin{lemma}
    \label{Lemma_RSSP}
    Consider a RSSP on the set $\mathbf{X}=(X_1,X_2,\dots,X_n)$ where $X_i$'s are sampled uniformly from $\{-M,\dots,-1,1,\dots,M\}$ with a target $t=\mathcal{O}(M)$. Let $Y_{n,t}$ be the number of possible solutions to the RSSP problem. Then
    \begin{equation*}
    \mathbb{P}(Y_{n,t}>0)\leq \begin{cases}
    \rho_n\left(\exp\left(-\frac{z^2}{2nM^2c_M}\right)+\mathcal{O}\left(\frac{1}{n}\right)\right)\;\;\;\;\;\text{if}\;\;z=0\\
        2\rho_n\left(\exp\left(-\frac{z^2}{2nM^2c_M}\right)+\mathcal{O}\left(\frac{1}{n}\right)\right)\;\;\;\text{if}\;\;z\neq 0,
    \end{cases}
\end{equation*}
\begin{equation*}
    \mathbb{P}(Y_{n,t}>0)\geq
        \frac{1}{\left(1+\exp\left(\frac{z^2}{nM^2c_M}\right)\left(\mathcal{O}\left(\frac{1}{n\rho_n}\right)+\mathcal{O}\left(\frac{1}{n}\right)\right)+\frac{1}{\rho_n}\right)},
\end{equation*}
where
    \begin{align*}
    z&=\Lambda-2t,  & \Lambda &= \sum_{i=1}^n X_i,\\
    \gamma_n &= \frac{1}{M\sqrt{2\pi n c_M}}, & c_M &= \mathbb{E}\left(\frac{X^2}{M^2}\right) = \frac{1}{3}+\frac{1}{2M}+\frac{1}{6M^2}.
    \end{align*}
\end{lemma}
\begin{proof}[Proof Sketch]
    We first convert the given RSSP to a RNPP through the transformation in Lemma \ref{lemma:npp_is_rssp}. The result then follows from Lemma \ref{Lemma:Upper_Lower_NPP}. See Appendix \ref{chap:npp_rssp} for details.
\end{proof}
The next lemma shows under what condition an RSSP can be solved with high probability.
\begin{lemma}
    \label{Lemma_Prob_Convergence}
    Let $M=M(n)$ be an arbitrary function of $n$. Consider as RSSP on the set $\mathbf{X}=(X_1,X_2,\dots,X_{n})$ sampled uniformly from $\{-M,\dots,-1,1,\dots,M\}$ with a target $t=\mathcal{O}(M)$. If
    \begin{equation*}
        \kappa_n = \lim_{n\to\infty}\frac{\log_2M}{n} <1,
    \end{equation*}
    then we have
    \begin{equation*}
        \mathbb{P}(Y_{n,t}>0) = 1-\mathcal{O}\left(\frac{1}{n^{\frac{1}{7}}}\right).
    \end{equation*}
\end{lemma}
\begin{proof}[Proof Sketch]
    It can be shown using Hoeffding's inequality (Theorem \ref{ineq:hoeff} in Appendix \ref{chap:ineq}), that with high probability the sum of all elements satisfies $\Lambda<\sqrt{\frac{2}{7}}M\sqrt{n\log n}$. Hence, the probability of solving a RSSP from Lemma \ref{Lemma_RSSP} can be analyzed under the assumption of $\kappa_n<1$. See Appendix \ref{chap:npp_rssp} for details.
\end{proof}

\section{SLTH and Weight Quantization}
    \label{sec:slth}
    \subsection{Notation and Setup}
\label{ssec:preliminaries}
In this subsection, we define some notation before stating our results. Scalars are denoted by lowercase letters such as $w$, $y$, etc. Vectors are represented by bold lowercase letters, e.g., $\mathbf{v}$, and the $i^{\text{th}}$ component of a vector $\mathbf{v}$ is denoted by $v_i$. Matrices are denoted by bold uppercase letters such as $\mathbf{M}$. If a matrix $\mathbf{W}$ has dimensions $d_1 \times d_2$, we write $\mathbf{W} \in \mathbb{R}^{d_1 \times d_2}$. We define the finite set $S_\delta := \{-1, -1 + \delta, -1 + 2\delta, \dots, 1\}$, where $\delta = 2^{-k}$ for some $k \in \mathbb{N}$. A real number $b$ is said to have precision $\delta$ if $b \in S_\delta$. We denote the $d$ -fold Cartesian product of $S_\delta$ by $S_\delta^d$; that is,
\begin{equation*}
    S_\delta^d := \underbrace{S_\delta \times \cdots \times S_\delta}_{d\ \text{times}}.
\end{equation*}
For $w \in S_\delta$ with $\delta = 2^{-k}$ and $\gamma = 2^{-m}$ such that $k > m$, we define the quantization operator $[\cdot]_{\gamma}$ by
\begin{equation*}
    [w]_{\gamma} := \frac{\lfloor w 2^m \rfloor}{2^m}.
\end{equation*}
This operation reduces the precision of $w$ to $\gamma$. For a vector $\mathbf{v}$, the notation $\mathbf{w} = [\mathbf{v}]_\gamma$ means $w_i = [v_i]_\gamma$ for all components $i$. We use $C, C_i$ for $i \in \mathbb{N}$ to denote positive absolute constants.
\begin{definition}
    An $\ell$-layer neural network is a function $f : \mathbb{R}^{d_0} \rightarrow \mathbb{R}^{d_\ell}$ defined as
    \begin{equation}
        f(\mathbf{x}) := \mathbf{W}_\ell \sigma(\mathbf{W}_{\ell-1} \cdots \sigma(\mathbf{W}_1 \mathbf{x})),
        \label{c4_nn}
    \end{equation}
    where $\mathbf{W}_i \in \mathbb{R}^{d_i \times d_{i-1}}$ for $i = 1, \dots, \ell$, $\mathbf{x} \in \mathbb{R}^{d_0}$, and $\sigma : \mathbb{R} \rightarrow \mathbb{R}$ is a nonlinear activation function. For a vector $\mathbf{x}$, the expression $\mathbf{v} = \sigma(\mathbf{x})$ denotes componentwise application: $v_i = \sigma(x_i)$.
\end{definition}

The entries of the matrices $\mathbf{W}_i$ are referred to as the weights or parameters of the network. In this work, we assume all activation functions are ReLU, i.e., $\sigma(x) = \max(0, x)$. This assumption is made for simplicity; the results can be extended to general activation functions as discussed in \cite{Burkholz}.

For a neural network $f(\mathbf{x}) = \mathbf{W}_\ell \sigma(\mathbf{W}_{\ell-1} \cdots \sigma(\mathbf{W}_1 \mathbf{x}))$, we refer to the quantity $\sigma(\mathbf{W}_k \cdots \sigma(\mathbf{W}_1 \mathbf{x}))$ as the output of the $k^{\text{th}}$ layer.

We next define some quantization strategies for neural networks which capture mixed-precision quantization practices. We defer the reader to the quantization paragraph in the Related Work (Section \ref{sec:rel}) for a discussion of such practices. 

\begin{definition}
    A \emph{$\delta$-quantized neural network} is a neural network whose weights are sampled uniformly from the set $S_\delta = \{-1, \dots,\delta, \dots, 1\}$, where $\delta = 2^{-k}$ for some $k \in \mathbb{N}$.
\end{definition}

\begin{definition}
    A neural network $f$ is called a $\gamma$-\emph{double mixed precision neural network} if the output of each layer is quantized to precision $\gamma$, i.e.,
    \begin{equation*}
        f(\mathbf{x}) = [\mathbf{W}_\ell [\sigma(\mathbf{W}_{\ell-1} \cdots [\sigma(\mathbf{W}_1 \mathbf{x})]_{\gamma})]_{\gamma}]_{\gamma}.
    \end{equation*}
    \label{def:prnn}
\end{definition}
\begin{definition}
    A neural network $f$ is called an $\gamma$-\emph{triple mixed precision neural network} if the outputs of its even-numbered layers are quantized to precision $\gamma$, i.e.,
    \begin{equation*}
        f(\mathbf{x}) = [\mathbf{W}_{2\ell} \sigma(\mathbf{W}_{2\ell-1} \cdots [\sigma(\mathbf{W}_2 (\sigma(\mathbf{W}_1 \mathbf{x})))]_{\gamma})]_{\gamma}.
    \end{equation*}
    \label{def:aprnn}
\end{definition}
More generally, a \emph{mixed-precision neural network} may reset the precision to $\gamma$ at some layers while leaving others unquantized. Reducing the precision of a $\delta$-quantized neural network $f$ to $\gamma$ means all weights of $f$ are mapped to $[\cdot]_\gamma$. We denote this operation as $[f]_\gamma$.

Our objective is to represent a \emph{target} Double Mixed Precision neural network $f$, with weights which are $\delta_1$ quantized, using a second, potentially overparameterized, mixed-precision network $g$ with finer quantization $\delta_2$, by quantizating and pruning it. For a neural network
\begin{equation*}
    g(\mathbf{x}) = \mathbf{M}_{2\ell} \sigma(\mathbf{M}_{2\ell-1} \cdots \sigma(\mathbf{M}_1 \mathbf{x})).
\end{equation*}
the pruned network $g_{\mathbf{S}_i}$ is defined as:
\begin{equation*}
    g_{\{\mathbf{S}_i\}}(\mathbf{x}) = (\mathbf{S}_{2\ell} \odot \mathbf{M}_{2\ell}) \sigma((\mathbf{S}_{2\ell-1} \odot \mathbf{M}_{2\ell-1}) \cdots \sigma((\mathbf{S}_1 \odot \mathbf{M}_1) \mathbf{x})),
\end{equation*}
where each $\mathbf{S}_i$ is a binary pruning mask with the same dimensions as $\mathbf{M}_i$, and $\odot$ denotes element-wise multiplication. The goal is to find masks $\mathbf{S}_1, \mathbf{S}_2, \dots, \mathbf{S}_\ell$ such that $f$ can be represented by the quantized and pruned version of $g$.
\subsection{Quantized SLTH Results}
Having discussed the previous work on NPP and it's connection to RSSP, we now apply these results to prove results on SLTH in quantized setting. The main question that we want to answer is the following: Suppose we are given a target neural network, whose weights are of precision $\delta_t$ and a large network whose weights are of precision $\delta_{\text{in}}$, such that $\delta_t\geq\delta_{\text{in}}$. Suppose we have the freedom to reduce the precision of the large network to $\delta$, and then we can prune it. What is the relationship between $\delta$ and size of the large network such that the bigger network can be pruned to the target network. 
Now we state our first main result, which is analogs to the theorem proved by \cite{pensia}, but in the quantized setting.
\begin{theorem}
    \label{thm:main}
    Let $\mathcal{F}$ be the class of $\delta_t$ quantized $\gamma$-double mixed Precision neural networks of the form 
    \begin{equation*}
    f(\mathbf{x}) = [\mathbf{W}_\ell [\sigma(\mathbf{W}_{\ell-1} \cdots [\sigma(\mathbf{W}_1 \mathbf{x})]_{\gamma})]_{\gamma}]_{\gamma}.
    \end{equation*}
    Consider a $2\ell$ layered randomly initialized $\delta_{\text{in}}$-quantized $\gamma$-double mixed Precision neural network
    \begin{equation*}
        g(\mathbf{x}) = [\mathbf{M}_{2\ell}\sigma(\mathbf{M}_{2\ell-1}...[\sigma(\mathbf{M}_2(\sigma(\mathbf{M}_1\mathbf{x})))]_{\gamma})]_{\gamma},
    \end{equation*}
    with ${\delta^2_{\text{in}}}\leq \delta_t$. Let $\delta^2_{\text{in}}\leq{\delta^2}\leq\delta_t$. Assume $\mathbf{M}_{2i}$ has dimension
    \begin{equation*}
        d_i \times C d_{i-1}\log_2 \frac{1}{\delta},
    \end{equation*}
    and $\mathbf{M}_{2i-1}$ has dimension
    \begin{equation*}
       C d_{i-1}\log_2 \frac{1}{\delta}  \times d_{i-1}.
    \end{equation*}
    Then the precision of elements of $\mathbf{M}_i$'s can be reduced to ${\delta}$, such that for every $f \in \mathcal{F}$,
    \begin{equation*}
        \exists\;\{\mathbf{S}_i\}_{i=1}^{2\ell}:\;\;\;\; [g_{\{\mathbf{S}_i\}}]_{\delta}(\mathbf{x}) =f(\mathbf{x}).
    \end{equation*}
    with probability at least
    \begin{equation*}
        1-N_t\;\mathcal{O}\left(\left(\log_2\frac{1}{\delta}\right)^{-\frac{1}{7}}\right)
    \end{equation*}
    where $N_t$ is the total number of parameters in $f$.
\end{theorem}
We prove the above theorem for a target network with a single weight (Lemma \ref{lemma_single_wt}) using the results on RSSP in the previous section, and then we give the idea for proving it in general. The proof is an application of the strategy in \cite{pensia} but with the use of Lemma \ref{lemma_single_wt}.  Details are given in Appendix \ref{chap:app_main}.
\begin{lemma}[Representing a single weight]
\label{lemma_single_wt}
        Let $g:\mathbb{R}\rightarrow \mathbb{R}$  be a randomly initialized $\delta_{\text{in}}$ quantized network of the form $g(x)=[\mathbf{v}^T\sigma(\mathbf{u}x)]_{\gamma}$ where $\mathbf{u},\mathbf{v},\in \mathbb{R}^{2n}$. Assume ${\delta^2_{\text{in}}}\leq \delta_t$ and $\delta^2_{\text{in}}\leq{\delta^2}\leq\delta_t$. Also assume $n>C\log_2\frac{1}{\delta}$. Then the precision of weights of $g$ can be reduced to ${\delta}$, such that with probability at least
    \begin{equation*}
        1-\mathcal{O}\left(\left(\log_2\frac{1}{\delta}\right)^{-\frac{1}{7}}\right),
    \end{equation*}
    we have for any $\;w\in S_{\delta_t}$
    \begin{equation*}
        \exists\; \mathbf{s}^1,\mathbf{s}^2\in \{0,1\}^{2n}:[wx]_{\gamma}= [[(\mathbf{v}\odot \mathbf{s}^2)^T\sigma(\mathbf{u}\odot \mathbf{s}^1)]_\delta (\mathbf{x})]_{\gamma}.
    \end{equation*}
\end{lemma}
\begin{proof}
Let the precision of $g$ be $\delta$. First decompose $wx = \sigma(wx)-\sigma(-wx)$. This is a general identity for ReLU non-linear activation and was introduced in \cite{slthproof}. WLOG \footnote{Without Loss of Generality} say $w>0$. Let
\begin{equation*}
    \mathbf{v} = \begin{bmatrix} \mathbf{b} \\ \mathbf{d}  \end{bmatrix}, \mathbf{u} = \begin{bmatrix} \mathbf{a} \\ \mathbf{c}  \end{bmatrix}, \mathbf{s}^1 = \begin{bmatrix} \mathbf{s}_1^1 \\ \mathbf{s}_2^1  \end{bmatrix},
    \mathbf{s}^2 = \begin{bmatrix} \mathbf{s}_1^2 \\ \mathbf{s}_2^2  \end{bmatrix},
\end{equation*}
where $\mathbf{a},\mathbf{b},\mathbf{c},\mathbf{d}\in \mathbb{R}^n, \mathbf{s}_1^1, \mathbf{s}_2^1,\mathbf{s}_1^2, \mathbf{s}_2^2 \in \{0,1\}^n$. This is shown diagrammatically in Figure \ref{fig:f1} in Appendix \ref{chap:figs}.

\vspace{0.2cm}
\textbf{Step 1:} Let $\mathbf{a}^+= \max\{\mathbf{0},\mathbf{a}\}$ be the vector obtained by pruning all the negative entries of $\mathbf{a}$. This is done by appropriately choosing $\mathbf{s}^1_1$ Since $w\geq 0$, then for all $x\leq 0$ we have $\sigma(wx)=\mathbf{b}^T\sigma(\mathbf{a}^+x)=0$. Moreover, further pruning of $\mathbf{a}^+$ would not affect this equality for $x\leq 0$. Thus we consider $x>0$ in next two steps. Therefore we get $\sigma (wx)=wx$ and $\mathbf{b}^T\mathbf{a}^+ x$ = $\sum_{i}b_ia_i^+x$.

\vspace{0.2cm}
\textbf{Step 2:} Consider the random variables $Z_i=b_ia_i^+$. These are numbers of precision $\delta^2$, sampled from the set $\{ab\;|\; a,b\in S_{{\delta}}\}$. Now $w$, which is a number of precision $\delta_t$, also belongs to the set $\{ab\;|\; a,b\in S_{\delta}\}$ because $\delta^2\leq\delta_t$. The numbers $\{Z_i\}$ are not distributed uniformly, but by a standard rejection sampling argument (as in \cite{Lueker1998}), there exists $C$ such that more that $2\log_2 \frac{1}{\delta}$ samples out of $C\log_2 \frac{1}{\delta}$ are uniform distributed. We prune the other samples such that we are left with $\bar{Z_i}$, which are uniformly distributed. Now by Lemma \ref{Lemma_Prob_Convergence}, as long as cardinality of $\{\bar{Z_i}\}$ is greater than $2\log_2\frac{1}{\delta}$, the Random Subset Sum Problem with set $\{\bar{Z_i}\}$ and target $w$ can be solved with probability atleast
    \begin{equation*}
        p \geq 1-\mathcal{O}\left(\left(\log_2\frac{1}{\delta}\right)^{-\frac{1}{7}}\right).
    \end{equation*}
Note that solving the Subset Sum Problem in an integer setting where numbers are sampled from $\{-M,\dots,M\}$ and solving it when numbers are sampled from $\{-1,\dots\delta,2\delta,\dots,1\}$ is equivalent (only difference is a scaling factor). In Lemma \ref{Lemma_RSSP} and \ref{Lemma_Prob_Convergence}, the sampling set is $\{-M,\dots,-1,1,\dots,M\}$, but $0$ can be rejected during rejection sampling. Hence it follows that with probability $p$
\begin{equation*}
    \forall\;w\in S_{\delta}^+,\;\;\; \exists\; \mathbf{s}_1\in \{0,1\}^{n}:w= \mathbf{b}^T\mathbf{s}_1\odot\mathbf{a}^+.
\end{equation*}
where $S_{\delta}^+$ denotes positive members of $S_\delta$. The part shown in green in Figure \ref{fig:f1} in Appendix \ref{chap:figs} hence handles positive inputs.

\vspace{0.2cm}
\textbf{Step 3:} Similar to steps 1 and 2, we can prune negative weights from $\mathbf{c}$ and let the red part shown in Figure \ref{fig:f1} in Appendix \ref{chap:figs} handle negative inputs. It will follow that with probability $p$
     \begin{equation*}
        \forall\;w\in S_{\delta}^+,\;\;\; \exists\; \mathbf{s}_2\in \{0,1\}^{n}:w= \mathbf{d}^T\mathbf{s}_2\odot\mathbf{c}^-.
    \end{equation*}
    Combining the two parts by union bound (Theorem \ref{eq:unbd}, Appendix \ref{chap:ineq}), Lemma 5 follows.
\end{proof}
\begin{proof}[Proof Sketch for Theorem \ref{thm:main}]
    The idea is to follow the strategy in \cite{pensia}. We represented a single weight in Lemma \ref{lemma_single_wt}. Similarly we can represent a neuron by representing each of its weights (shown explicitly in Lemma \ref{lemma_single_neuron} and diagrammatically in Figure \ref{fig:f2} in Appendix \ref{chap:app_main}). Using the representation of a single neuron, we represent a full layer (shown explicitly in Lemma \ref{lemma_layer} and diagrammatically in Figure \ref{fig:f3} in Appendix \ref{chap:app_main}). Then we represent a full network by applying Lemma \ref{lemma_layer} layer by layer. See Appendix \ref{chap:app_main} for details.
\end{proof}
Our next result employs construction from \cite{Burkholz}.
\begin{theorem}
    \label{thm:main_2}
    Let $\mathcal{F}$ be the class of $\delta_t$ quantized $\gamma$-double mixed Precision neural networks of the form
    \begin{equation*}
        f(\mathbf{x}) = [\mathbf{W}_\ell [\sigma(\mathbf{W}_{\ell-1} \cdots [\sigma(\mathbf{W}_1 \mathbf{x})]_{\gamma})]_{\gamma}]_{\gamma}.
    \end{equation*}
    Consider an $\ell+1$ layered randomly initialized $\gamma$-mixed precision resetting network which resets the precession to $\gamma$ in all layers except the first one,
    \begin{equation*}
        g(\mathbf{x}) = [\mathbf{M}_{2\ell}\sigma[(\mathbf{M}_{2\ell-1}...[\sigma(\mathbf{M}_2(\sigma(\mathbf{M}_1\mathbf{x})))]_{\gamma})]_\gamma]_{\gamma},
    \end{equation*}
    whose weights are sampled from $\{-1\dots,-\delta,\delta,\dots,1\}$ with ${\delta_{\text{in}}}\leq \delta_t$. Let $\delta_{\text{in}}\leq{\delta}\leq\delta_t$. If $\mathbf{M}_{1}$ and $\mathbf{M}_{2}$ have dimensions
    \begin{align*}
        d_0 \times C d_{0}\log_2 \frac{1}{\delta} &&\text{and}&& d_1\log_2\frac{1}{\delta} \times C d_{0}\log_2\frac{1}{\delta}
    \end{align*}
    respectively, $\mathbf{M}_{i+1}$ has dimension greater than
    \begin{equation*}
       d_{i}\log_2\frac{1}{\delta}\times d_{i+1}\log_2\frac{1}{\delta}
    \end{equation*}
    $\forall\; 2<i<l-1$ and $M_{\ell+1}$ has dimension greater than
    \begin{equation*}
        \log_2\left(\frac{1}{\delta}\right)d_{l-1}\times d_l.
    \end{equation*}
    Then the precision elements of $\mathbf{M}_i$'s can be reduced to $\delta$ such that for every $f \in \mathcal{F}$ we have
    \begin{equation*}
        \exists\;\{\mathbf{S}_i\}_{i=1}^{\ell+1}:\;\;\;\; [g_{\{\mathbf{S}_i\}}]_{\delta}(\mathbf{x}) =f(\mathbf{x}).
    \end{equation*}
    with probability atleast
    \begin{equation*}
        1-N_t\;\log_2\left(\frac{1}{\delta}\right)\;\mathcal{O}\left(\left(\log_2\frac{1}{\delta}\right)^{-\frac{1}{7}}\right)
    \end{equation*}
    where $N_t$ is the total number of parameters in $f$.
\end{theorem}
\begin{proof}[Proof Sketch for Theorem \ref{thm:main_2}]
    We follow the construction in \cite{Burkholz}. The idea is to use the same trick as the previous result to represent a layer, but to copy it many times. Hence the representation of a layer which was supposed to give output $(x_1,x_2,\dots,x_N)$, will give output $(x_1, x_1,\dots,x_2,x_2,\dots,x_N,x_N\dots ,x_N)$. These copies can now be used while representing the next layer, without adding an intermediate layer in between (shown in Lemma \ref{Bruk_lemma_layer} and diagrammatically in Figure \ref{fig:f4} in Appendix \ref{chap:app_main}).
\end{proof}

\subsection{Lower Bound by Parameter Counting}

In this section we show by a parameter counting argument, akin to that employed in \cite{pensia,nataleSparsityStrongLottery2024}, that there exists a two layered $\delta$-quantized network with $d^2$ parameters that cannot be represented by a neural network unless it has $\Omega\left(d^2\log_2\left(\frac{1}{\delta}\right)\right)$ parameters. Note that any linear transformation $\mathbf{Wx}$ where $\mathbf{W}\in S_\delta^d\times S_\delta^d$ and $\mathbf{x}\in S_\delta^d$ can be expressed as a 2 layered neural network. Let $\mathcal{F}$ be the class of functions
\begin{align}
    \mathcal{F}:=\{h_{\mathbf{W}}:\mathbf{W}\in S_\delta^d\times S_\delta^d\},&&\text{where}&&h_{\mathbf{W}}(\mathbf{x})=\begin{bmatrix}
        \mathbf{I}&-\mathbf{I}\end{bmatrix} \sigma\left(\begin{bmatrix}
        \mathbf{W}\\-\mathbf{W}\end{bmatrix}\mathbf{x}\right).
    \label{eq:linear_fu}
\end{align}

\begin{theorem}
\label{thm:lower}
    Let $g:\mathbb{R}^d\to\mathbb{R}^d$ be a $\delta$ quantized neural network of the form 
    \begin{equation*}
        g(\mathbf{x}) = \mathbf{M}_{\ell}\sigma(\mathbf{M}_{\ell-1}...\sigma(\mathbf{M}_1\mathbf{x})),
    \end{equation*}
    where elements of $\mathbf{M}_i$'s are sampled from arbitrary distributions over $S_{\delta}$. Let $\mathcal{G}$ be the set of all matrices that can be formed by pruning $g$. Let $\mathcal{F}$ be defined as in Eq. \ref{eq:linear_fu}. If
    \begin{equation*}
        \forall \;h\in \mathcal{F},\;\mathbb{P}\left(\exists\;g'\in \mathcal{G}:g'=h\right)\geq p,
    \end{equation*}
    then the total number of non zero parameters of $g$ is at least
    \begin{equation*}
        \log_2p+d^2\log_2\left(\frac{2}{\delta}+1\right).
    \end{equation*}
\end{theorem}
\begin{proof}[Proof Sketch]
    Theorem \ref{thm:lower} follows from a parameter counting argument. We simply count the number of different functions in $\mathcal{F}$ and demand that with probability $p$, any $f\in \mathcal{F}$ be represented by pruning $g$. See Appendix \ref{chap:app_main} for details.
\end{proof}
The following immediate corollary of the previous theorem provide a matching lower bound to Theorem \ref{thm:main}.
\begin{corollary}
    \label{cor}
    If $g$ is a two-layer network satisfying the hypothesis of Theorem \ref{thm:lower}, then its width is $\Omega(d\log \frac 1\delta)$.
\end{corollary}
    
\section{Conclusion}
    \label{sec:con}
    We have proved \emph{optimal} over–parameterization bounds for the Strong Lottery Ticket Hypothesis (SLTH) in the finite-precision setting.  Specifically, we showed that any $\delta_t$-quantized target network~$N_t$ can be recovered \emph{exactly} by pruning a larger, randomly–initialized network~$N_{\mathrm{in}}$ with precision $\delta_{\text{in}}$.  By reducing the pruning task to a \emph{quantized} Random Subset Sum instance and importing the sharp phase-transition analysis for the Number Partitioning Problem, we derived width requirements that match the information-theoretic lower bound up to absolute constants. 
These results not only close the gap between upper and lower bounds for quantized SLTH, but also certify, for the first time, that pruning alone can yield \emph{exact} finite-precision subnetworks rather than merely approximate ones. Beyond their theoretical interest, our findings pinpoint the precise interplay between quantization granularity and over-parameterization, and they suggest that mixed-precision strategies may enjoy similarly tight guarantees. An immediate open problem is to generalize our techniques to structured architectures—most notably convolutional, residual, and attention-based networks—where weight sharing and skip connections introduce additional combinatorial constraints.  
Another interesting direction is to incorporate layer-wise mixed precision and to analyze the robustness of lottery tickets under stochastic quantization noise, which is of interest for practical deployment on low-precision hardware accelerators. We believe that the combinatorial perspective adopted here will prove equally effective in these broader settings, ultimately advancing our theoretical understanding of extreme model compression.

\bibliographystyle{plainnat}
\bibliography{references}

\newpage

\appendix

    \section{NPP and RSSP Results}
    \label{chap:npp_rssp}
    We start by stating the result by \cite{Borgs2001} on NPP. Define $I_{n,z}$ as
\begin{equation}
    Z_{n,z} = 2^n I_{n,z} \times \begin{cases}
        1\;\; \text{if} \;z=0\\
        2\;\; \text{if} \;z>0.
    \end{cases}
    \label{eq:ZLrel}
\end{equation}
\begin{theorem}
\label{thm:Borgs_main}
Let $C_0 > 0$ be a finite constant, let $M = M(n)$ be an arbitrary function of $n$, let
\begin{align*}
    \gamma_n = \frac{1}{M\sqrt{2\pi n c_M}}&& \text{where} &&c_M = \mathbb{E}\left(\frac{X^2}{M^2}\right) = \frac{1}{3}+\frac{1}{2M}+\frac{1}{6M^2}~,
\end{align*}
and let $z$ and $z'$ be integers. Then,
\begin{equation*}
    \mathbb{E}[I_{n,z}] = \gamma_n\left(\exp\left(-\frac{z^2}{2nM^2c_M}\right)+O(n^{-1})\right).
\end{equation*}
Furthermore
\begin{multline*}
    \mathbb{E}[I_{n,z}I_{n,z'}] = 2\gamma_n^2\left(\exp\left(-\frac{z^2+(z')^2}{2nM^2c_M}\right)+\mathcal{O}\left(\frac{1}{n}\right)+\mathcal{O}\left(\frac{1}{n\gamma_n 2^n}\right)\right) \\ +\frac{\gamma_n}{2^n}\left( \delta_{z+z',0}+\delta_{z-z',0}\right)\exp\left(-\frac{z^2+(z')^2}{2nM^2c_M}\right)
\end{multline*}    
if $z$ and $z'$ are of the same parity, i.e., both odd or both even, while $\mathbb{E}[I_{n,z}I_{n,z'}]=0$ if $z$ and $z'$ are of different parity.
\end{theorem}
Now using Theorem \ref{thm:Borgs_main}, we prove Lemma \ref{Lemma:Upper_Lower_NPP}
\begin{proof}[Proof of Lemma \ref{Lemma:Upper_Lower_NPP}]
Consider $z\neq 0$. 
From Theorem \ref{thm:Borgs_main} we have
\begin{equation*}
    \mathbb{E}[I_{n,z}] = \gamma_n\left(\exp\left(-\frac{z^2}{2nM^2c_M}\right)+\mathcal{O}(n^{-1})\right).
\end{equation*}
If we multiply by $2^{n+1}$ we get
\begin{equation}
    \mathbb{E}[Z_{n,z}] = \rho_n\left(\exp\left(-\frac{z^2}{2nM^2c_M}\right)+\mathcal{O}(n^{-1})\right).
    \label{eqEZNL}
\end{equation}
It also follows from the above equation that
\begin{equation}
    \mathbb{E}[Z_{n,z}] \geq \rho_n\exp\left(-\frac{z^2}{2nM^2c_M}\right).
    \label{eqEZNLa}
\end{equation}
Furthermore, from Theorem \ref{thm:Borgs_main} we have
\begin{multline*}
    \mathbb{E}[I_{n,z}I_{n,z'}] = 2\gamma_n^2\left(\exp\left(-\frac{z^2+(z')^2}{2nM^2c_M}\right)+\mathcal{O}\left(\frac{1}{n}\right)+\mathcal{O}\left(\frac{1}{n\gamma_n 2^n}\right)\right) \\ +\frac{\gamma_n}{2^n}\left( \delta_{z+z',0}+\delta_{z-z',0}\right)\exp\left(-\frac{z^2+(z')^2}{2nM^2c_M}\right)
\end{multline*}
If $z=z'$ we get
\begin{equation*}
    \mathbb{E}[I_{n,z}^2] = 2\gamma_n^2\left(\exp\left(-\frac{2z^2}{2nM^2c_M}\right)+\mathcal{O}\left(\frac{1}{n}\right)+\mathcal{O}\left(\frac{1}{n\gamma_n 2^n}\right)\right) +\frac{\gamma_n}{2^n}\exp\left(-\frac{2z^2}{2nM^2c_M}\right)
\end{equation*}
Multiplying by $(2^{n+1})^2$ we get
\begin{equation}
    \mathbb{E}[Z_{n,z}^2] = 2\rho_n^2\left(\exp\left(-\frac{2z^2}{2nM^2c_M}\right)+\mathcal{O}\left(\frac{1}{n}\right)+\mathcal{O}\left(\frac{1}{n\rho_n}\right)\right) +2\rho_n\exp\left(-\frac{2z^2}{2nM^2c_M}\right)
    \label{EZNL2}
\end{equation}
Now using Markov's inequality (Theorem \ref{ineq:markov}, Appendix \ref{chap:ineq}) and Eq. \ref{eqEZNL} we get
\begin{equation*}
    \mathbb{P}(Z_{n,z}>0) \leq \rho_n\left(\exp\left(-\frac{z^2}{2nM^2c_M}\right)+\mathcal{O}\left(\frac{1}{n}\right)\right).
\end{equation*}
Using Cauchy-Schwartz inequality (Theorem \ref{ineq:CS}, Appendix \ref{chap:ineq}) Eq. \ref{eqEZNLa}  and Eq. \ref{EZNL2} we thus get
\begin{align*}
    \mathbb{P}(Z_{n,z}>0) &\geq \frac{\rho_n^2\exp\left(-\frac{2z^2}{2nM^2c_M}\right)}{2\rho_n^2\left(\exp\left(-\frac{2z^2}{2nM^2c_M}\right)+\mathcal{O}\left(\frac{1}{n}\right)+\mathcal{O}\left(\frac{1}{n\rho_n}\right)\right) + 2\rho_n\exp\left(-\frac{2z^2}{2nM^2c_M}\right)}\\
     \Rightarrow \mathbb{P}(Z_{n,z}>0)&\geq \frac{1}{2\left(1+\exp\left(\frac{z^2}{nM^2c_M}\right)\left(\mathcal{O}\left(\frac{1}{n\rho_n}\right)+\mathcal{O}\left(\frac{1}{n}\right)\right)+\frac{1}{\rho_n}\right)}.
\end{align*}
The same calculation can be done for $z=0$, the only difference is that $Z_{n,z}=2^nI_{n,z}$ (Eq. \ref{eq:ZLrel}).
\end{proof}
Before moving ahead, lets establish the equivalence of NPP and SSP by proving Lemma  \ref{lemma:npp_is_rssp}.
\begin{proof}[Proof of Lemma  \ref{lemma:npp_is_rssp}]
    First of all notice that a NPP on the set $\mathbf{X} = (X_1, X_2,\dots,X_n)$ where $X_i$'s are sampled uniformly from $\{-M,\dots,-1,1,\dots,M\}$ can be solved iff the NPP on the set $\mathbf{X} = (|X_1|, |X_2|,\dots,|X_n|)$ sampled uniformly from $\{1,2,\dots,M\}$ can be solved. This is because, first, it is obvious that $\{X_i\}_{i=1}^n$ is distributed uniformly over $\{1,2,\dots,M\}$, and secondly, the NPP does not care about the signs of the numbers, a sign can always be absorbed in the $\sigma_i$ while solving the NPP.
    
    We have an SPP with set $\mathbf{X}$, sampled uniformly from $\{-M,\dots,-1,1,\dots,M\}$ and target $t$. Assume number partitioning problem can be solved, given the set $\mathbf{X}$ and target $\Lambda-2t$. Notice that NPP does not care about the sign of the target, as an NPP with target $k$ can be solved iff that NPP with target $-k$ can be solved. Assume there exists two partitions $S_1$ and $S_2$ of $\mathbf{X}$, with $S_1$ summing to $x$ and $S_2$ summing to $\Lambda-x$, such that $\sum_{i\in S_2}X_i-\sum_{j\in S_1} X_j = (\Lambda-x) -x= \Lambda-2t$, which is equivalent to $\sum_{j\in S_1} j=x=t$. Hence, $S_1$ sums up to $t$, so the given SSP can be solved. The reverse direction also follows from the argument, proving the result.
\end{proof}
\begin{proof}[Proof of Lemma \ref{Lemma_RSSP}]
    Considers the number partitioning problem corresponding to the given random subset sum problem (Lemma \ref{lemma:npp_is_rssp}). The target of this number partitioning problem is  $z=\Lambda-2t$. Consider $z\neq0$. A key observation here is if $\Lambda$ is even (event denoted by $\mathscr{E}_n$), then $z$ is also even and if $\Lambda$ is odd (event denoted by $\mathscr{O}_n$), then $z$ is also odd. The probability that the random subset sum problem can be solved can be written in terms of the probability that the number partitioning problem can be solved
    \begin{equation*}
        \mathbb{P}(Y_{n,t}>0) = \mathbb{P}(\mathscr{E}_n)\mathbb{P}(Z_{n,z}>0|\mathscr{E}_n)+\mathbb{P}(\mathscr{O}_n)\mathbb{P}(Z_{n,z}>0|\mathscr{O}_n)
    \end{equation*}
    Since on $\mathscr{E}_n$, $z$ is always even and on $\mathscr{O}_n$, $z$ is always odd, we have two cases. If $z$ is even, then
    \begin{equation*}
         \mathbb{P}(Z_{n,z}>0)= \mathbb{P}(\mathscr{E}_n)\mathbb{P}(Z_{n,z}>0|\mathscr{E}_n).
    \end{equation*}
    If $z$ is odd, then
    \begin{equation*}
        \mathbb{P}(Z_{n,z}>0)= \mathbb{P}(\mathscr{O}_n)\mathbb{P}(Z_{n,z}>0|\mathscr{O}_n).
    \end{equation*}
 Hence $\mathbb{P}(Y_{n,t}>0)$ can be written as
    \begin{equation*}
        \mathbb{P}(Y_{n,t}>0) = 2\mathbb{P}(Z_{n,z}>0).
    \end{equation*}
    From Lemma \ref{Lemma:Upper_Lower_NPP} it follows that
    \begin{equation*}
    \mathbb{P}(Y_{n,t}>0)\leq 2\rho_n\left(\exp\left(-\frac{z^2}{2nM^2c_M}\right)+\mathcal{O}\left( \frac{1}{n}\right)\right),
\end{equation*}
\begin{equation*}
    \mathbb{P}(Y_{n,t}>0)\geq \frac{1}{\left(1+\exp\left(\frac{z^2}{nM^2c_M}\right)\left(\mathcal{O}\left(\frac{1}{n\rho_n}\right)+\mathcal{O}\left(\frac{1}{n}\right)\right)+\frac{1}{\rho_n}\right)}.
\end{equation*}
Same can be done for $z=0$, only difference is a factor of 2.
\end{proof}
\begin{proof}[Proof of Lemma \ref{Lemma_Prob_Convergence}]
    We are given that $\lim_{n\to\infty}\kappa_n$ exists and is less than 1. Consider a more sensitive parametrization
    \begin{align*}
        \kappa_n = 1-\frac{\log_2n}{2n}+\frac{\lambda_n}{n}&&\text{or}&&M=\frac{2^{n+\lambda_n}}{\sqrt{n}}.
    \end{align*}
    In this parametrization $\lim_{n\to\infty}\kappa_n<1$ means $\lim_{n\to\infty}\lambda_n\to -\infty$. Note that in this regime $\rho_n\to\infty$. Now we have
    \begin{align*}
    \mathbb{P}(Y_{n,t}>0)&\geq \frac{1}{\left(1+\exp\left(\frac{(\Lambda-2t)^2}{nM^2c_M}\right)\left(\mathcal{O}\left(\frac{1}{n\rho_n}\right)+\mathcal{O}\left(\frac{1}{n}\right)\right)+\frac{1}{\rho_n}\right)}.
\end{align*}
Now $t=\mathcal{O}(M)$ and demand $\Lambda$ as
\begin{equation*}
    \Lambda<\frac{1}{\sqrt{3+\beta}}M\sqrt{n\log n}.
\end{equation*}
According to Hoeffding's inequality (Theorem \ref{ineq:hoeff}, Appendix \ref{chap:ineq}), that happens with probability
\begin{align*}
    \mathbb{P}\left(\Lambda<\frac{1}{\sqrt{3+\beta}}M\sqrt{n\log n}\right) &\geq 1-\exp\left(\frac{\frac{1}{{3+\beta}}M^2{n\log n}}{4nM^2}\right)\\
    \Rightarrow \mathbb{P}\left(\Lambda<\frac{1}{\sqrt{3+\beta}}M\sqrt{n\log n}\right) &\geq 1-\frac{1}{n^{\frac{1}{2(3+\beta)}}}.
\end{align*}
Now as $\rho_n\to\infty$, we have
\begin{align*}
    \mathbb{P}(Y_{n,t}>0)&\geq 1-\mathcal{O}\left(\frac{1}{n^{\frac{\beta}{3+\beta}}}\right).
\end{align*}
Let $\mathbb{P}(E)$ be the probability of events ${P}(Y_{n,t}>0)$ and $\Lambda<\frac{1}{\sqrt{3+\beta}}M\sqrt{n\log n}$ happening together. Then by union bound (Theorem \ref{eq:unbd}, Appendix \ref{chap:ineq}) we can say that
\begin{equation*}
    \mathbb{P}(E) \geq 1-\mathcal{O}\left(\frac{1}{n^{\frac{\beta}{3+\beta}}}\right)-\frac{1}{n^{\frac{1}{2(3+\beta)}}}.
\end{equation*}
Note that this probability will converge to 1 fastest if $\beta=\frac{1}{2}$. Hence we choose $\beta=\frac{1}{2}$ and we get
\begin{equation*}
    \mathbb{P}(E) = \left(1-\mathcal{O}\left(\frac{1}{n^{\frac{1}{7}}}\right)\right).
\end{equation*}
\end{proof}

    \section{SLTH-Quantization Results}
    \label{chap:app_main}
    In this appendix, we prove the results related to SLTH and weight quantization. We start by proving Theorem \ref{thm:main}. The idea is to follow the strategy of \cite{pensia}, but use Lemma \ref{lemma_single_wt}.
\begin{lemma}[Representing a single Neuron]
     Consider a randomly initialized $\delta_{\text{in}}$ quantized neural network of the form $g(\bm{x})=[\mathbf{v}^T\sigma(\mathbf{M}\bm{x})]_{\gamma}$ with $\bm{x}\in \mathbb{R}^d$. Assume ${\delta^2_{\text{in}}}\leq \delta_t$ and $\delta^2_{\text{in}}\leq{\delta^2}\leq\delta_t$. Let $f_{\mathbf{w}}(\bm{x})=[\mathbf{w}^T\bm{x}]_{\gamma}$ be a single layered $\delta_t$ quantized network. Let $\mathbf{M}\in \mathbb{R}^{Cd\log_2 \frac{1}{\delta}\times d}$ and $\mathbf{v}\in \mathbb{R}^{Cd\log_2 \frac{1}{\delta}}$. Then the precision of weights of $g$ can be reduced to ${\delta}$, such that with probability atleast
     \begin{equation*}
        1-d\;\mathcal{O}\left(\left(\log_2\frac{1}{\delta}\right)^{-\frac{1}{7}}\right),
    \end{equation*}
    we have
    \begin{equation*}
        \forall \mathbf{w} \in S_\delta^{d}\;\;\exists\;\;\mathbf{s},\mathbf{T}:\;f_{\mathbf{w}}(\bm{x})=[g_{\{\mathbf{s},\mathbf{T}\}}]_{\delta}(\mathbf{x}),
    \end{equation*}
    where $[g_{\{\mathbf{s},\mathbf{T}\}}]_{\delta}(\mathbf{x})$ is the pruned network for a choice of binary vector $\mathbf{s}$ and matrix $\mathbf{T}$,
    \label{lemma_single_neuron}
\end{lemma}
\begin{proof} Assume weights of $g$ are of precision $\delta$. We prove the required results by representing each weight of the neuron using Lemma \ref{lemma_single_wt} (See Figure \ref{fig:f2}, Appendix \ref{chap:figs}).\newline
\noindent \textbf{Step 1:} We first prune $\mathbf{M}$ to create a block-diagonal matrix $\mathbf{M}'$. Specifically, we create M by only keeping the following non-zero entries:
    \begin{equation*}
        \begin{bmatrix}
            \mathbf{u}_1 & 0&\cdots&0\\
             0 & \mathbf{u}_2&\cdots&0\\
             \vdots & \vdots & \ddots & \vdots \\
             0 & 0&\cdots&\mathbf{u}_d
        \end{bmatrix},\;\;\;\;\;\;\text{where}\;\; \mathbf{u}_i\in \mathbb{R}^{C\log_2 \frac{1}{\delta}}.
    \end{equation*}
    We choose the binary matrix $\mathbf{T}$ to be such that $\mathbf{M}'=\mathbf{T}\odot \mathbf{M}$. We also decompose $\mathbf{v}$ and $\mathbf{s}$ as
    \begin{equation*}
        \mathbf{s} = \begin{bmatrix}
            \mathbf{s}_1 \\
             \mathbf{s}_2\\
             \vdots\\
             \mathbf{s}_d
        \end{bmatrix},\;\;\;\;\mathbf{v} = \begin{bmatrix}
            \mathbf{v}_1 \\
             \mathbf{v}_2\\
             \vdots\\
             \mathbf{v}_d
        \end{bmatrix},\;\;\text{where}\;\; \mathbf{s}_i,\mathbf{v}_i\in \mathbb{R}^{C\log_2 \frac{1}{\delta}}.
    \end{equation*}
    \textbf{Step 2:} Consider the event
    \begin{equation*}
        E_i :\;\; [w_ix_i]= [(\mathbf{v}_i\odot \mathbf{s}_i)^T\sigma(\mathbf{u_i}x_i)].
    \end{equation*}
    According to Lemma \ref{lemma_single_wt}, this event happens with probability 
    \begin{equation*}
        p = 1-\mathcal{O}\left(\left(\log_2\frac{1}{\delta}\right)^{-\frac{1}{7}}\right).
    \end{equation*}
    The event (say $E$) in the assumption of Lemma \ref{lemma_single_neuron} corresponds with the intersection of these events $E=\cap_{i=1}^d E_i$. By taking a union bound (Theorem \ref{eq:unbd}, Appendix \ref{chap:ineq}), $E$ happens with a probability $dp-(d-1)$, which is equal to
    \begin{equation*}
        1-d\;\mathcal{O}\left(\left(\log_2\frac{1}{\delta}\right)^{-\frac{1}{7}}\right).
    \end{equation*}
    The process is illustrated in Figure \ref{fig:f2}. Note that we want $>\log_2(\frac{1}{\delta})$ samples to be assured that a RSSP is solved with high probability, but we include that in the constant $C$. Any extra factors (a factor of 2 for example) is also absorbed in $C$ throughout the proof.
\end{proof}
\begin{lemma}[Representing a single layer]
    \label{lemma_layer}
    Consider a randomly initialized $\delta_{\text{in}}$ quantized two layer neural network of the form $g(\mathbf{x})=[\mathbf{N}\sigma(\mathbf{Mx})]_{\gamma}$ with $x\in \mathbb{R}^{d_1}$. Assume ${\delta^2_{\text{in}}}\leq \delta_t$ and $\delta^2_{\text{in}}\leq{\delta^2}\leq\delta_t$.  Let $f_{\mathbf{W}}(\mathbf{x})=[\mathbf{W}\mathbf{x}]_{\gamma}$ be a single layered $\delta_t$ quantized network. Assume $\mathbf{N}$ has dimension $d_2\times C d_1\log_2\frac{1}{\delta}$ and $\mathbf{M}$ has dimension $Cd_1\log_2 \frac{1}{\delta}\times d_1$. Then the precision of weights of $g$ can be reduced to to ${\delta}$, such that with probability atleast
    \begin{equation*}
        1-d_1d_2\;\mathcal{O}\left(\left(\log_2\frac{1}{\delta}\right)^{-\frac{1}{7}}\right),
    \end{equation*}
    \begin{equation*}
         \forall \;\mathbf{W} \in S_{\delta_1}^{d_1\times d_2}\;\;\exists\;\;\mathbf{S},\mathbf{T}:\;f_{\mathbf{W}}(\mathbf{x})=[g_{\{\mathbf{S},\mathbf{T}\}}]_{\delta}(\mathbf{x}),
     \end{equation*}
      where $[g_{\{\mathbf{S},\mathbf{T}\}}]_{\delta}(\mathbf{x})$ is the pruned network for a choice of pruning matrices $\mathbf{S}$ and $\mathbf{T}$.
\end{lemma}
\begin{proof}
    Assume weights of $g$ are of precision $\delta$. We first prune $\mathbf{M}$ to get a block diagonal matrix $\mathbf{M}'$
    \begin{equation*}
        \mathbf{M}'=\begin{bmatrix}
            \mathbf{u}_1 & 0&\cdots&0\\
             0 & \mathbf{u}_2&\cdots&0\\
             \vdots & \vdots & \ddots & \vdots \\
             0 & 0&\cdots&\mathbf{u}_{d_1}
        \end{bmatrix},\;\;\;\;\;\;\text{where}\;\; \mathbf{u}_i\in \mathbb{R}^{C\log_2 \frac{1}{\delta}}.
    \end{equation*}
    Thus, $\mathbf{T}$ is such that $\mathbf{M}'=\mathbf{T}\odot \mathbf{M}$. We also decompose $\mathbf{N}$ and $\mathbf{S}$ as following
    \begin{equation*}
    \mathbf{S}=\begin{bmatrix}
            \mathbf{s}_{1,1}^T &\cdots&\mathbf{s}_{1,d_1}^T\\
            \mathbf{s}_{2,1}^T &\cdots&\mathbf{s}_{2,d_1}^T\\
             \vdots & \ddots & \vdots \\
             \mathbf{s}_{d_2,1}^T &\cdots&\mathbf{s}_{d_2,d_1}^T
        \end{bmatrix},\;\;\;\;\mathbf{N}=\begin{bmatrix}
            \mathbf{v}_{1,1}^T &\cdots&\mathbf{v}_{1,d_1}^T\\
            \mathbf{v}_{2,1}^T &\cdots&\mathbf{v}_{2,d_1}^T\\
             \vdots & \ddots & \vdots \\
             \mathbf{v}_{d_2,1}^T &\cdots&\mathbf{v}_{d_2,d_1}^T
        \end{bmatrix},\;\;\;\;\text{where}\; \mathbf{v}_{i,j}, \mathbf{s}_{i,j}\in \mathbb{R}^{C\log_2 \frac{1}{\delta}}.
    \end{equation*}
    Now note that pruning $\mathbf{u}_i$ and $\mathbf{v}_{i,j}$ (using $\mathbf{s}_{i,j}$) is equivalent to Lemma \ref{lemma_single_neuron}. Hence it's simply an application of Lemma \ref{lemma_single_wt} $d_1d_2$ times. Hence the event in assumption of Lemma \ref{lemma_layer} occurs with a probability $d_1d_2p-(d_1d_2-1)$, by a union bound (Theorem \ref{eq:unbd}, Appendix \ref{chap:ineq}), which is equal to
    \begin{equation*}
        1-d_1d_2\;\mathcal{O}\left(\left(\log_2\frac{1}{\delta}\right)^{-\frac{1}{7}}\right).
    \end{equation*}
    The process is illustrated in Figure \ref{fig:f3}, Appendix \ref{chap:figs}. Note that we want $>\log_2(\frac{1}{\delta})$ samples to be assured that a RSSP is solved with high probability, but we include that in the constant $C$. Constant Factors also absorbed in $C$.
\end{proof}

\begin{proof}[Proof of Theorem \ref{thm:main}.]
    Now we can see that Theorem \ref{thm:main} can be proved by applying Lemma \ref{lemma_layer} layer wise, where two layers of the large network represent one layer of the target. Note that the precision is set of $\delta_1$ after every layer (of the large network) and precision is set of $\delta_1$ after every layer (of the target network). Let the total number of parameters in the target network be $N_t$, i.e.,
    \begin{equation*}
        N_t=\sum_{i=1}^{l-1}d_id_{i+1}.
    \end{equation*}
    Then the event in assumption of Theorem \ref{thm:main}, by union bound (Theorem \ref{eq:unbd}, Appendix \ref{chap:ineq}), occurs with a probability $N_tp-(N_t-1)$, where which is equal to
    \begin{equation*}
        1-N_t\;\mathcal{O}\left(\left(\log_2\frac{1}{\delta}\right)^{-\frac{1}{7}}\right).
    \end{equation*}
\end{proof}

\subsection*{This construction improves the depth}
In this subsection, we adapt construction by \cite{Burkholz} to prove Theorem \ref{thm:main_2}. The process is illustrated in Figure \ref{fig:f4}, Appendix \ref{chap:figs}.
\begin{lemma}
    \label{Bruk_lemma_layer}
    Consider a randomly initialized $\delta_{\text{in}}$ quantized two layered neural network $g(\mathbf{x})=[\mathbf{N}\sigma(\mathbf{Mx})]_{\gamma}$ with $\mathbf{x}\in \mathbb{R}^{d_1}$, whose weights are sampled uniformly from $\{-1,\dots,-\delta,\delta,\dots,1\}$. Assume ${\delta^2_{\text{in}}}\leq \delta_t$ and $\delta^2_{\text{in}}\leq{\delta^2}\leq\delta_t$. Let 
    \begin{equation*}
        f_{\mathbf{W}}(\mathbf{x})=\begin{bmatrix}
            [\mathbf{W}\mathbf{x}]_{\gamma}\\
            [\mathbf{W}\mathbf{x}]_{\gamma}\\
            \vdots\\
            [\mathbf{W}\mathbf{x}]_{\gamma}
        \end{bmatrix}
    \end{equation*} 
    be a single layered $\delta_t$ quantized network where $\mathbf{W}\mathbf{x}$ is repeated $\log_2(\frac{1}{\delta})$ times and $\mathbf{W}$ has dimension $d_1\times d_2$. If $\mathbf{N}$ has dimension $d_2\log_2\frac{1}{\delta}\times C d_1\log_2\frac{1}{\delta}$ and $\mathbf{M}$ has dimension $Cd_1\log_2 \frac{1}{\delta}\times d_1$. Then the precision of weights of $g$ can be reduced to to $\delta$, such that with probability
    \begin{equation*}
        1-d_1d_2\log_2\left(\frac{1}{\delta}\right)\;\mathcal{O}\left(\left(\log_2\frac{1}{\delta}\right)^{-\frac{1}{7}}\right).
    \end{equation*}
    we have
    \begin{equation*}
         \forall\;\mathbf{W} \in S_\delta^{d_1\times d_2}\;\;\exists\;\;\mathbf{S},\mathbf{T}:\;f_{\mathbf{W}}(\mathbf{x})=[g_{\{\mathbf{S},\mathbf{T}\}}]_{\delta}(\mathbf{x}),
     \end{equation*}
     where $[g_{\{\mathbf{S},\mathbf{T}\}}]_{\delta}(\mathbf{x})$ is the pruned network for a choice of pruning matrices $\mathbf{S}$ and $\mathbf{T}$.
\end{lemma}

\begin{proof}
    Assume weights of $g$ are of precision $\delta$. We first prune $\mathbf{M}$ to get a block diagonal matrix $\mathbf{M}'$
    \begin{equation*}
        \mathbf{M}'=\begin{bmatrix}
            \mathbf{u}_1 & 0&\cdots&0\\
             0 & \mathbf{u}_2&\cdots&0\\
             \vdots & \vdots & \ddots & \vdots \\
             0 & 0&\cdots&\mathbf{u}_{d_1}
        \end{bmatrix},\;\;\;\;\;\;\text{where}\;\; \mathbf{u}_i\in \mathbb{R}^{C\log_2 \frac{1}{\delta}}.
    \end{equation*}
    Thus, $\mathbf{T}$ is such that $\mathbf{M}'=\mathbf{T}\odot \mathbf{M}$. We also decompose $\mathbf{N}$ and $\mathbf{S}$ as following
    \begin{align*}
    \mathbf{S}&= 
    \begin{bmatrix}
    \mathbf{S}_1\\
    \mathbf{S}_2\\
    \vdots\\
    \mathbf{S}_{\log_2\left(\frac{1}{\delta}\right)}
    \end{bmatrix}&
    \mathbf{N}&= 
    \begin{bmatrix}
    \mathbf{N}_1\\
    \mathbf{N}_2\\
    \vdots\\
    \mathbf{N}_{\log_2\left(\frac{1}{\delta}\right)}
    \end{bmatrix}
    \end{align*}
    where
    \begin{align*}
    \mathbf{S}_k&=\begin{bmatrix}
            (\mathbf{s}_{1,1}^T)_k
            &\cdots&(\mathbf{s}_{1,d_1}^T)_k\\
            (\mathbf{s}_{2,1}^T)_k &\cdots&(\mathbf{s}_{2,d_1}^T)_k\\
             \vdots & \ddots & \vdots \\
             (\mathbf{s}_{d_2,1}^T)_k &\cdots&(\mathbf{s}_{d_2,d_1}^T)_k
        \end{bmatrix},&
        \mathbf{N}&=\begin{bmatrix}
            (\mathbf{v}_{1,1}^T)_k &\cdots&(\mathbf{v}_{1,d_1}^T)_k\\
            (\mathbf{v}_{2,1}^T)_k &\cdots&(\mathbf{v}_{2,d_1}^T)_k\\
             \vdots & \ddots & \vdots \\
             (\mathbf{v}_{d_2,1}^T)_k &\cdots&(\mathbf{v}_{d_2,d_1}^T)_k
        \end{bmatrix},\\
        &\text{and}\; (\mathbf{v}_{i,j})_k, (\mathbf{s}_{i,j})_k\in \mathbb{R}^{C\log_2 \frac{1}{\delta}}.
    \end{align*}
    Now note that pruning $\mathbf{u}_i$ and $(\mathbf{v}_{i,j})_k$ (using $(\mathbf{s}_{i,j})_k$) is equivalent to Lemma \ref{lemma_single_neuron}. Hence it's simply an application of Lemma \ref{lemma_single_wt} $d_1d_2\log_2\left(\frac{1}{\delta}\right)$ times. Hence the event in assumption of Lemma \ref{Bruk_lemma_layer} occurs with a probability 
    \begin{equation*}
        1-d_1d_2\log_2\left(\frac{1}{\delta}\right)\;\mathcal{O}\left(\left(\log_2\frac{1}{\delta}\right)^{-\frac{1}{7}}\right),
    \end{equation*}
    using the union bound (Theorem \ref{eq:unbd}, Appendix \ref{chap:ineq}).
\end{proof}
\begin{proof} [Proof of Theorem \ref{thm:main_2}]
    In Lemma \ref{Bruk_lemma_layer} we represented the first layer of the target network, with a difference that output contains many copies. The rest of the proof is same is \cite{Burkholz}. These copies can be used to represent weights in the next layer. The argument follows iteratively for all layers until we reach the last layer, where copying is not required. The only key difference is that rejection sampling is not required, giving the required size free of any undetermined constants. The process is illustrated in Figure \ref{fig:f3}. The event in the assumption of Theorem \ref{thm:main_2} happens with probability
    \begin{equation*}
        1-N_t\;\log_2\left(\frac{1}{\delta}\right)\;\mathcal{O}\left(\left(\log_2\frac{1}{\delta}\right)^{-\frac{1}{7}}\right)
    \end{equation*}
\end{proof}

\subsection*{Lower Bound by Parameter Counting}
Here we prove Theorem \ref{thm:lower} which follows by a parameter counting in the discrete setting.
\begin{proof}[Proof of Theorem \ref{thm:lower}]
    Two matrices represent the same function iff all their elements are the same. Therefore, the number of functions in $\mathcal{F}$ is
    \begin{equation*}
        \left(\frac{2}{\delta}+1\right)^{d^2}.
    \end{equation*}
    Let the number of non zero parameters in $g$ be $\alpha$, then the number of functions in $\mathcal{G}$ is $2^\alpha$. Now for the assumption of Theorem \ref{thm:lower} to hold, we must have
    \begin{align*}
        2^\alpha&\geq p\left(\frac{2}{\delta}+1\right)^{d^2}\\
        \implies \alpha &\geq \log_2p+d^2\log_2\left(\frac{2}{\delta}+1\right).
    \end{align*}
\end{proof}
Corollary \ref{cor} is an immediate consequence of Theorem \ref{thm:lower}.

    \section{Inequalities}
    \label{chap:ineq}
    \begin{theorem}
    For a non-negative, integer-valued random variable $X$ we have
\begin{equation*}
    \mathbb{P}(X>0)\leq\mathbb{E}[X].
\end{equation*}
    \label{ineq:markov}
\end{theorem}
\begin{theorem}
    If $X>0$ is a random variable with finite variance, then
    \begin{equation*}
    \mathbb{P}(X>0)\geq \frac{(\mathbb{E}[X])^2}{\mathbb{E}[X^2]}.
    \end{equation*}
    \label{ineq:CS}
\end{theorem}
\begin{theorem}
Let $X_1,X_2,\dots,X_n$ be independent random variables such that $a_i\leq X_i\leq b_i$ almost surely. Consider the sum of these random variables,
\begin{equation*}
    S_n = X_1+X_2+\cdots+X_n.
\end{equation*}
Then Hoeffding's theorem states that, for all $t>0$,
\begin{equation*}
    \mathbb{P}(S_n-\mathbb{E}(s_n)\geq t) \leq \exp\left(-\frac{2t^2}{\sum_{i=1}^{n}(b_i-a_i)^2}\right)
\end{equation*}
    \label{ineq:hoeff}
\begin{equation*}
    \mathbb{P}(|S_n-\mathbb{E}(s_n)|\geq t) \leq 2\exp\left(-\frac{2t^2}{\sum_{i=1}^{n}(b_i-a_i)^2}\right).
\end{equation*}
\end{theorem}
\begin{theorem}
    For any events $A_1,A_2,\dots,A_n$ we have
\begin{equation*}
    \mathbb{P}\left(\bigcap_{i=1}^nA_i\right)\geq \max\left(0,\;\sum_{i=1}^{n}\mathbb{P}(A_i)-(n-1)\right).
\end{equation*}
    \label{eq:unbd}
\end{theorem}
\newpage

    \section{Figures}
    \label{chap:figs}
    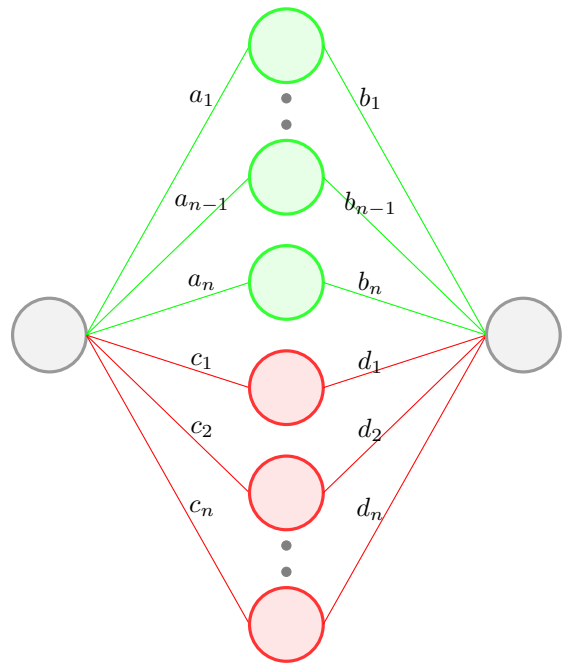
\begin{figure}[ht]
    \centering
    \begin{tikzpicture}[scale=0.7]
    \filldraw[color=gray!80, fill=gray!10, very thick](-4.5,3) circle (0.7);
    \draw[red] (-3.8,3) -- (-0.7,-2.5);
    \draw[red] (-3.8,3) -- (-0.7,0);
    \draw[red] (-3.8,3) -- (-0.7,2);
    \draw[green] (-3.8,3) -- (-0.7,4);
    \draw[green] (-3.8,3) -- (-0.7,6);
    \draw[green] (-3.8,3) -- (-0.7,8.5);
    \node at (-1.6,7.5) {$a_1$};
    \node at (-1.6,5.5) {$a_{n-1}$};
    \node at (-1.6,4) {$a_{n}$};
    \node at (-1.6,2.5) {$c_1$};
    \node at (-1.6,1.2) {$c_2$};
    \node at (-1.6,-0.3) {$c_n$};
    \filldraw[color=red!80, fill=red!10, very thick](0,-2.5) circle (0.7);
    \filldraw[color=red!80, fill=red!10, very thick](0,0) circle (0.7);
    \filldraw[color=red!80, fill=red!10, very thick](0,2) circle (0.7);
    \filldraw[color=green!80, fill=green!10, very thick](0,4) circle (0.7);
    \filldraw[color=green!80, fill=green!10, very thick](0,6) circle (0.7);
    \filldraw[color=green!80, fill=green!10, very thick](0,8.5) circle (0.7);
    \draw[red] (0.7,-2.5) -- (3.8,3);
    \draw[red] (0.7,0) -- (3.8,3);
    \draw[red] (0.7,2) -- (3.8,3);
    \draw[green] (0.7,4) -- (3.8,3);
    \draw[green] (0.7,6) -- (3.8,3);
    \draw[green] (0.7,8.5) -- (3.8,3);
    \node at (1.6,7.5) {$b_1$};
    \node at (1.6,5.5) {$b_{n-1}$};
    \node at (1.6,4) {$b_{n}$};
    \node at (1.6,2.5) {$d_1$};
    \node at (1.6,1.2) {$d_2$};
    \node at (1.6,-0.3) {$d_n$};
    \filldraw[color=gray!80, fill=gray!10, very thick](4.5,3) circle (0.7);
    \filldraw [gray] (0,-1.5) circle (2.5pt);
    \filldraw [gray] (0,-1) circle (2.5pt);
    \filldraw [gray] (0,7) circle (2.5pt);
    \filldraw [gray] (0,7.5) circle (2.5pt);
    \end{tikzpicture}
    \caption{Approximating a single weight with ReLU activation (\cite{pensia}): The network shown in the figure represents a single weight after pruning.}
    \label{fig:f1}
\end{figure}
\newpage
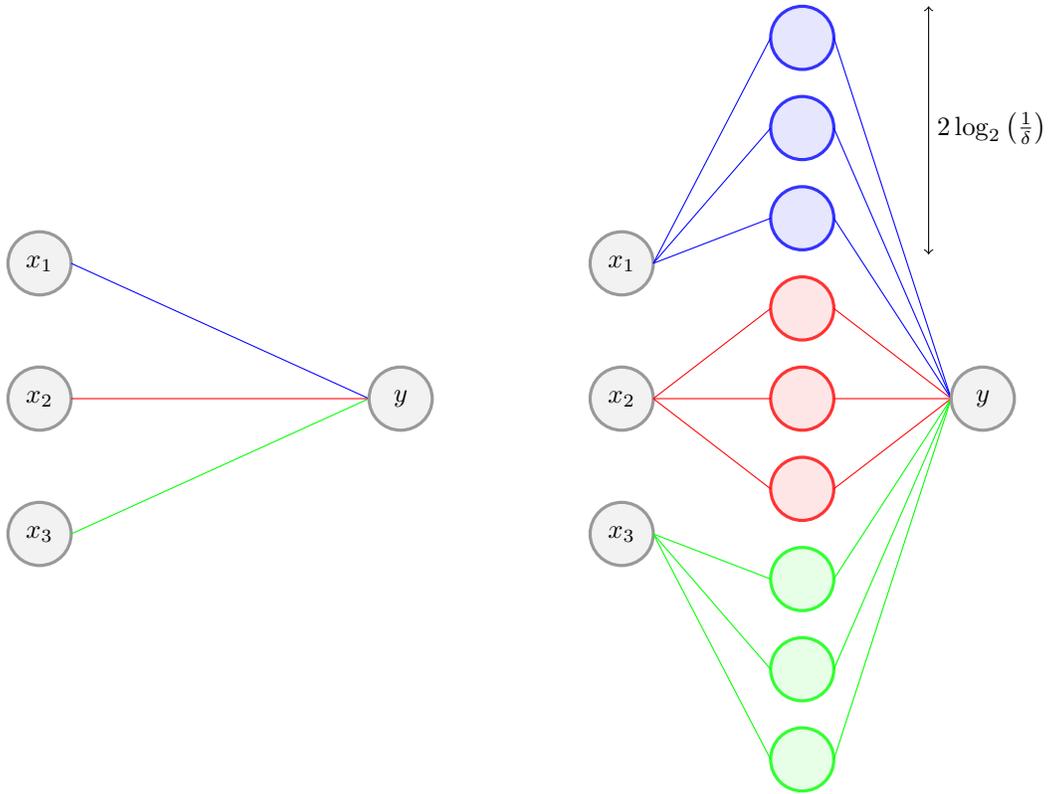
\begin{figure}[ht]
    \centering
    \begin{tikzpicture}
    \node at (0,0) {\begin{tikzpicture}[scale=0.6]
    \filldraw[color=gray!80, fill=gray!10, very thick](0,0) circle (0.7);
    \filldraw[color=gray!80, fill=gray!10, very thick](0,3) circle (0.7);
    \filldraw[color=gray!80, fill=gray!10, very thick](0,-3) circle (0.7);
    
    \draw[green] (0.7,-3) -- (7.3,0);
    \draw[red] (0.7,0) -- (7.3,0);
    \draw[blue] (0.7,3) -- (7.3,0);
    \filldraw[color=gray!80, fill=gray!10, very thick](8,0) circle (0.7);
    \node at (8,0) {$y$};
    \node at (0,0) {$x_2$};
    \node at (0,3) {$x_1$};
    \node at (0,-3) {$x_3$};
\end{tikzpicture}};
    \node at (8,0) {\begin{tikzpicture}[scale=0.6]
    \filldraw[color=gray!80, fill=gray!10, very thick](0,0) circle (0.7);
    \filldraw[color=gray!80, fill=gray!10, very thick](0,3) circle (0.7);
    \filldraw[color=gray!80, fill=gray!10, very thick](0,-3) circle (0.7);
    \filldraw[color=red!80, fill=red!10, very thick](4,-2) circle (0.7);
    \filldraw[color=red!80, fill=red!10, very thick](4,0) circle (0.7);
    \filldraw[color=red!80, fill=red!10, very thick](4,2) circle (0.7);
    \draw[red] (0.7,0) -- (3.3,-2);
    \draw[red] (0.7,0) -- (3.3,0);
    \draw[red] (0.7,0) -- (3.3,2);
    \draw[red] (4.7,-2) -- (7.3,0);
    \draw[red] (4.7,0) -- (7.3,0);
    \draw[red] (4.7,2) -- (7.3,0);

    \filldraw[color=green!80, fill=green!10, very thick](4,-4) circle (0.7);
    \filldraw[color=green!80, fill=green!10, very thick](4,-6) circle (0.7);
    \filldraw[color=green!80, fill=green!10, very thick](4,-8) circle (0.7);
    \draw[green] (0.7,-3) -- (3.3,-4);
    \draw[green] (0.7,-3) -- (3.3,-6);
    \draw[green] (0.7,-3) -- (3.3,-8);
    \draw[green] (4.7,-4) -- (7.3,0);
    \draw[green] (4.7,-6) -- (7.3,0);
    \draw[green] (4.7,-8) -- (7.3,0);

    \filldraw[color=blue!80, fill=blue!10, very thick](4,4) circle (0.7);
    \filldraw[color=blue!80, fill=blue!10, very thick](4,6) circle (0.7);
    \filldraw[color=blue!80, fill=blue!10, very thick](4,8) circle (0.7);
    \draw[blue] (0.7,3) -- (3.3,4);
    \draw[blue] (0.7,3) -- (3.3,6);
    \draw[blue] (0.7,3) -- (3.3,8);
    \draw[blue] (4.7,4) -- (7.3,0);
    \draw[blue] (4.7,6) -- (7.3,0);
    \draw[blue] (4.7,8) -- (7.3,0);

    \filldraw[color=gray!80, fill=gray!10, very thick](8,0) circle (0.7);
    \node at (8,0) {$y$};
    \node at (0,0) {$x_2$};
    \node at (0,3) {$x_1$};
    \node at (0,-3) {$x_3$};
    \draw[<->][black] (6.8,8.7) -- (6.8,3.2);
    \node at (8.2,6) {$2\log_2\left(\frac{1}{\delta}\right)$};
    \end{tikzpicture}};
\end{tikzpicture}
    \caption{Representing a single neuron (\cite{pensia}): The figure on the left shows the target network, where as Figure on the right shows the large network. The colors indicate which part in the target is represented by which part of the source. For example, the red weight on the left is represented by the red subnetwork on the right.}
    \label{fig:f2}
\end{figure}
\newpage
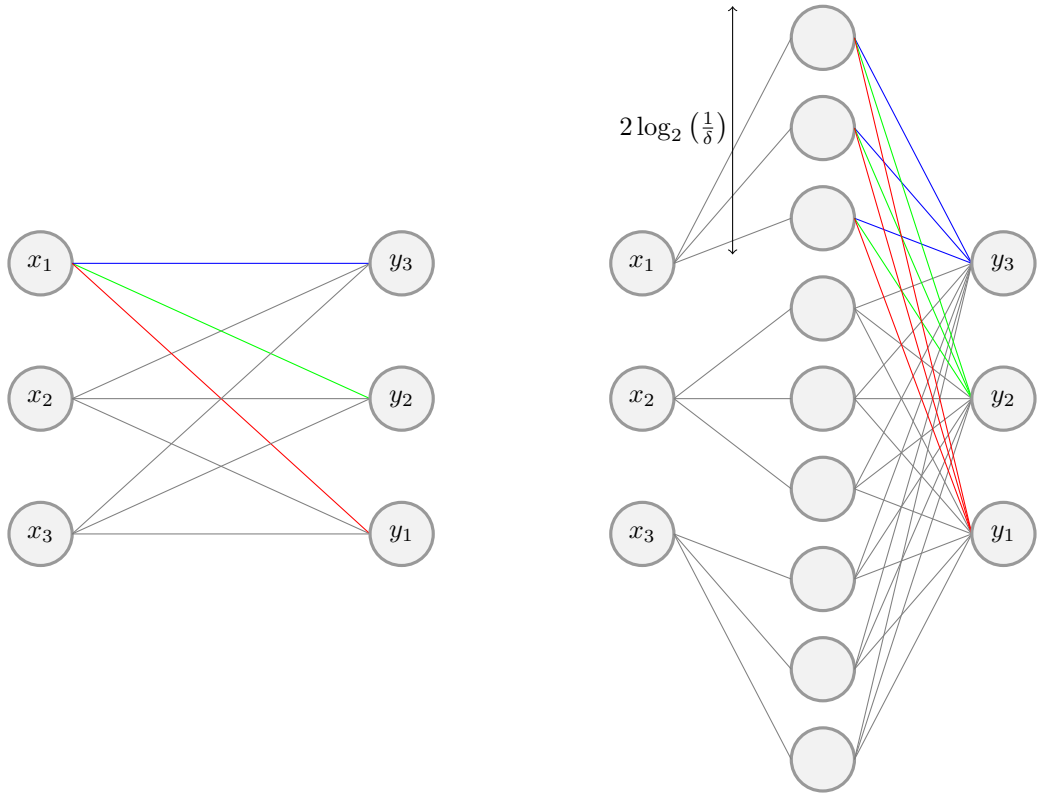
\begin{figure}[ht]
    \centering
    \begin{tikzpicture}
    \node at (0,0) {\begin{tikzpicture}[scale=0.6]
    \filldraw[color=gray!80, fill=gray!10, very thick](0,0) circle (0.7);
    \filldraw[color=gray!80, fill=gray!10, very thick](0,3) circle (0.7);
    \filldraw[color=gray!80, fill=gray!10, very thick](0,-3) circle (0.7);

    \draw[gray] (0.7,0) -- (7.3,0);
    \draw[gray] (0.7,0) -- (7.3,3);
    \draw[gray] (0.7,0) -- (7.3,-3);

    \draw[green] (0.7,3) -- (7.3,0);
    \draw[blue] (0.7,3) -- (7.3,3);
    \draw[red] (0.7,3) -- (7.3,-3);

    \draw[gray] (0.7,-3) -- (7.3,0);
    \draw[gray] (0.7,-3) -- (7.3,3);
    \draw[gray] (0.7,-3) -- (7.3,-3);
    
    \filldraw[color=gray!80, fill=gray!10, very thick](8,0) circle (0.7);
    \filldraw[color=gray!80, fill=gray!10, very thick](8,3) circle (0.7);
    \filldraw[color=gray!80, fill=gray!10, very thick](8,-3) circle (0.7);
    \node at (0,0) {$x_2$};
    \node at (0,3) {$x_1$};
    \node at (0,-3) {$x_3$};

    \node at (8,0) {$y_2$};
    \node at (8,3) {$y_3$};
    \node at (8,-3) {$y_1$};
    \end{tikzpicture}};
    \node at (8,0) {\begin{tikzpicture}[scale=0.6]
    \filldraw[color=gray!80, fill=gray!10, very thick](0,0) circle (0.7);
    \filldraw[color=gray!80, fill=gray!10, very thick](0,3) circle (0.7);
    \filldraw[color=gray!80, fill=gray!10, very thick](0,-3) circle (0.7);
    \filldraw[color=gray!80, fill=gray!10, very thick](4,-2) circle (0.7);
    \filldraw[color=gray!80, fill=gray!10, very thick](4,0) circle (0.7);
    \filldraw[color=gray!80, fill=gray!10, very thick](4,2) circle (0.7);
    \draw[gray] (0.7,0) -- (3.3,-2);
    \draw[gray] (0.7,0) -- (3.3,0);
    \draw[gray] (0.7,0) -- (3.3,2);
    \draw[gray] (4.7,-2) -- (7.3,0);
    \draw[gray] (4.7,0) -- (7.3,0);
    \draw[gray] (4.7,2) -- (7.3,0);

    \draw[gray] (4.7,-2) -- (7.3,-3);
    \draw[gray] (4.7,0) -- (7.3,-3);
    \draw[gray] (4.7,2) -- (7.3,-3);

    \draw[gray] (4.7,-2) -- (7.3,3);
    \draw[gray] (4.7,0) -- (7.3,3);
    \draw[gray] (4.7,2) -- (7.3,3);

    \filldraw[color=gray!80, fill=gray!10, very thick](4,-4) circle (0.7);
    \filldraw[color=gray!80, fill=gray!10, very thick](4,-6) circle (0.7);
    \filldraw[color=gray!80, fill=gray!10, very thick](4,-8) circle (0.7);
    \draw[gray] (0.7,-3) -- (3.3,-4);
    \draw[gray] (0.7,-3) -- (3.3,-6);
    \draw[gray] (0.7,-3) -- (3.3,-8);
    \draw[gray] (4.7,-4) -- (7.3,-3);
    \draw[gray] (4.7,-6) -- (7.3,-3);
    \draw[gray] (4.7,-8) -- (7.3,-3);

    \draw[gray] (4.7,-4) -- (7.3,0);
    \draw[gray] (4.7,-6) -- (7.3,0);
    \draw[gray] (4.7,-8) -- (7.3,0);

    \draw[gray] (4.7,-4) -- (7.3,3);
    \draw[gray] (4.7,-6) -- (7.3,3);
    \draw[gray] (4.7,-8) -- (7.3,3);

    \filldraw[color=gray!80, fill=gray!10, very thick](4,4) circle (0.7);
    \filldraw[color=gray!80, fill=gray!10, very thick](4,6) circle (0.7);
    \filldraw[color=gray!80, fill=gray!10, very thick](4,8) circle (0.7);
    \draw[gray] (0.7,3) -- (3.3,4);
    \draw[gray] (0.7,3) -- (3.3,6);
    \draw[gray] (0.7,3) -- (3.3,8);
    \draw[blue] (4.7,4) -- (7.3,3);
    \draw[blue] (4.7,6) -- (7.3,3);
    \draw[blue] (4.7,8) -- (7.3,3);

    \draw[green] (4.7,4) -- (7.3,0);
    \draw[green] (4.7,6) -- (7.3,0);
    \draw[green] (4.7,8) -- (7.3,0);

    \draw[red] (4.7,4) -- (7.3,-3);
    \draw[red] (4.7,6) -- (7.3,-3);
    \draw[red] (4.7,8) -- (7.3,-3);


    \filldraw[color=gray!80, fill=gray!10, very thick](8,0) circle (0.7);
    \filldraw[color=gray!80, fill=gray!10, very thick](8,3) circle (0.7);
    \filldraw[color=gray!80, fill=gray!10, very thick](8,-3) circle (0.7);
    \node at (0,0) {$x_2$};
    \node at (0,3) {$x_1$};
    \node at (0,-3) {$x_3$};

    \node at (8,0) {$y_2$};
    \node at (8,3) {$y_3$};
    \node at (8,-3) {$y_1$};
    \draw[<->][black] (2,8.7) -- (2,3.2);
    \node at (0.7,6) {$2\log_2\left(\frac{1}{\delta}\right)$};
\end{tikzpicture}};
\end{tikzpicture}
    \caption{Representing a layer (\cite{pensia}): Figure on the left shows the target network, where as Figure on the right shows the large network. The colors indicate which part in the target is represented by which part of the source. For example, the red weight on the left is represented by the red weights on the right.}
    \label{fig:f3}
\end{figure}
\newpage
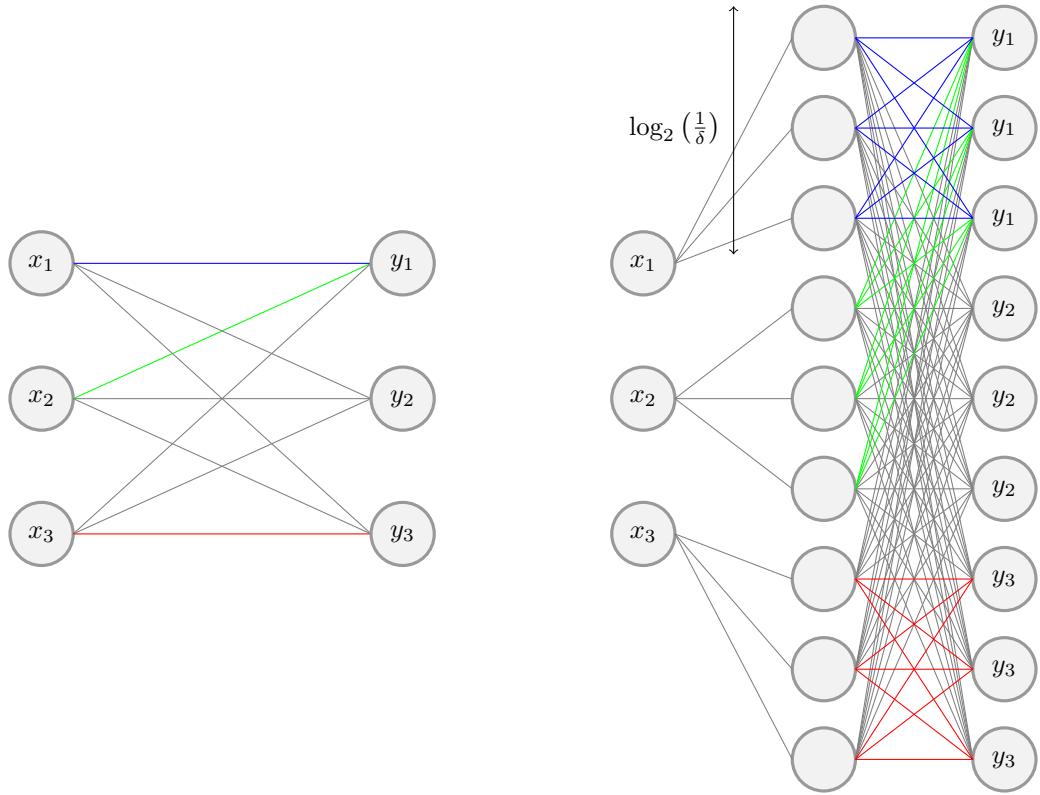
\begin{figure}[ht]
    \centering
    \begin{tikzpicture}
    \node at (0,0) {\begin{tikzpicture}[scale=0.6]
    \filldraw[color=gray!80, fill=gray!10, very thick](0,0) circle (0.7);
    \node at (0,0) {$x_2$};
    \filldraw[color=gray!80, fill=gray!10, very thick](0,3) circle (0.7);
    \node at (0,3) {$x_1$};
    \filldraw[color=gray!80, fill=gray!10, very thick](0,-3) circle (0.7);
    \node at (0,-3) {$x_3$};

    \draw[gray] (0.7,0) -- (7.3,0);
    \draw[green] (0.7,0) -- (7.3,3);
    \draw[gray] (0.7,0) -- (7.3,-3);

    \draw[gray] (0.7,3) -- (7.3,0);
    \draw[blue] (0.7,3) -- (7.3,3);
    \draw[gray] (0.7,3) -- (7.3,-3);

    \draw[gray] (0.7,-3) -- (7.3,0);
    \draw[gray] (0.7,-3) -- (7.3,3);
    \draw[red] (0.7,-3) -- (7.3,-3);
    
    \filldraw[color=gray!80, fill=gray!10, very thick](8,0) circle (0.7);
    \node at (8,0) {$y_2$};
    \filldraw[color=gray!80, fill=gray!10, very thick](8,3) circle (0.7);
    \node at (8,3) {$y_1$};
    \filldraw[color=gray!80, fill=gray!10, very thick](8,-3) circle (0.7);
    \node at (8,-3) {$y_3$};
    \end{tikzpicture}};
    \node at (8,0) {\begin{tikzpicture}[scale=0.6]
    \filldraw[color=gray!80, fill=gray!10, very thick](0,0) circle (0.7);
    \filldraw[color=gray!80, fill=gray!10, very thick](0,3) circle (0.7);
    \filldraw[color=gray!80, fill=gray!10, very thick](0,-3) circle (0.7);
    \filldraw[color=gray!80, fill=gray!10, very thick](4,-2) circle (0.7);
    \filldraw[color=gray!80, fill=gray!10, very thick](4,0) circle (0.7);
    \filldraw[color=gray!80, fill=gray!10, very thick](4,2) circle (0.7);
    \draw[gray] (0.7,0) -- (3.3,-2);
    \draw[gray] (0.7,0) -- (3.3,0);
    \draw[gray] (0.7,0) -- (3.3,2);

    \filldraw[color=gray!80, fill=gray!10, very thick](4,-4) circle (0.7);
    \filldraw[color=gray!80, fill=gray!10, very thick](4,-6) circle (0.7);
    \filldraw[color=gray!80, fill=gray!10, very thick](4,-8) circle (0.7);
    \draw[gray] (0.7,-3) -- (3.3,-4);
    \draw[gray] (0.7,-3) -- (3.3,-6);
    \draw[gray] (0.7,-3) -- (3.3,-8);

    \filldraw[color=gray!80, fill=gray!10, very thick](4,4) circle (0.7);
    \filldraw[color=gray!80, fill=gray!10, very thick](4,6) circle (0.7);
    \filldraw[color=gray!80, fill=gray!10, very thick](4,8) circle (0.7);
    \draw[gray] (0.7,3) -- (3.3,4);
    \draw[gray] (0.7,3) -- (3.3,6);
    \draw[gray] (0.7,3) -- (3.3,8);


    \filldraw[color=gray!80, fill=gray!10, very thick](8,-8) circle (0.7);
    \filldraw[color=gray!80, fill=gray!10, very thick](8,-6) circle (0.7);
    \filldraw[color=gray!80, fill=gray!10, very thick](8,-4) circle (0.7);
    \filldraw[color=gray!80, fill=gray!10, very thick](8,-2) circle (0.7);
    \filldraw[color=gray!80, fill=gray!10, very thick](8,0) circle (0.7);
    \filldraw[color=gray!80, fill=gray!10, very thick](8,2) circle (0.7);
    \filldraw[color=gray!80, fill=gray!10, very thick](8,4) circle (0.7);
    \filldraw[color=gray!80, fill=gray!10, very thick](8,6) circle (0.7);
    \filldraw[color=gray!80, fill=gray!10, very thick](8,8) circle (0.7);

    \draw[gray] (4.7,-8) -- (7.3,0);
    \draw[gray] (4.7,-8) -- (7.3,-2);
    \draw[gray] (4.7,-8) -- (7.3,2);
    \draw[gray] (4.7,-8) -- (7.3,4);
    \draw[gray] (4.7,-8) -- (7.3,6);
    \draw[gray] (4.7,-8) -- (7.3,8);

    \draw[gray] (4.7,0) -- (7.3,2);
    \draw[gray] (4.7,0) -- (7.3,0);
    \draw[gray] (4.7,0) -- (7.3,-2);
    \draw[gray] (4.7,0) -- (7.3,-4);
    \draw[gray] (4.7,0) -- (7.3,-6);
    \draw[gray] (4.7,0) -- (7.3,-8);

    \draw[gray] (4.7,2) -- (7.3,2);
    \draw[gray] (4.7,2) -- (7.3,0);
    \draw[gray] (4.7,2) -- (7.3,-2);
    \draw[gray] (4.7,2) -- (7.3,-4);
    \draw[gray] (4.7,2) -- (7.3,-6);
    \draw[gray] (4.7,2) -- (7.3,-8);

    \draw[gray] (4.7,4) -- (7.3,2);
    \draw[gray] (4.7,4) -- (7.3,0);
    \draw[gray] (4.7,4) -- (7.3,-2);
    \draw[gray] (4.7,4) -- (7.3,-4);
    \draw[gray] (4.7,4) -- (7.3,-6);
    \draw[gray] (4.7,4) -- (7.3,-8);

    \draw[gray] (4.7,6) -- (7.3,2);
    \draw[gray] (4.7,6) -- (7.3,0);
    \draw[gray] (4.7,6) -- (7.3,-2);
    \draw[gray] (4.7,6) -- (7.3,-4);
    \draw[gray] (4.7,6) -- (7.3,-6);
    \draw[gray] (4.7,6) -- (7.3,-8);

    \draw[gray] (4.7,8) -- (7.3,2);
    \draw[gray] (4.7,8) -- (7.3,0);
    \draw[gray] (4.7,8) -- (7.3,-2);
    \draw[gray] (4.7,8) -- (7.3,-4);
    \draw[gray] (4.7,8) -- (7.3,-6);
    \draw[gray] (4.7,8) -- (7.3,-8);
    
    \draw[gray] (4.7,-6) -- (7.3,0);
    \draw[gray] (4.7,-6) -- (7.3,-2);
    \draw[gray] (4.7,-6) -- (7.3,2);
    \draw[gray] (4.7,-6) -- (7.3,4);
    \draw[gray] (4.7,-6) -- (7.3,6);
    \draw[gray] (4.7,-6) -- (7.3,8);

    \draw[gray] (4.7,-4) -- (7.3,0);
    \draw[gray] (4.7,-4) -- (7.3,-2);
    \draw[gray] (4.7,-4) -- (7.3,2);
    \draw[gray] (4.7,-4) -- (7.3,4);
    \draw[gray] (4.7,-4) -- (7.3,6);
    \draw[gray] (4.7,-4) -- (7.3,8);

    \draw[gray] (4.7,-2) -- (7.3,0);
    \draw[gray] (4.7,-2) -- (7.3,-2);
    \draw[gray] (4.7,-2) -- (7.3,2);
    \draw[gray] (4.7,-2) -- (7.3,-4);
    \draw[gray] (4.7,-2) -- (7.3,-6);
    \draw[gray] (4.7,-2) -- (7.3,-8);

    \draw[red] (4.7,-8) -- (7.3,-8);
    \draw[red] (4.7,-8) -- (7.3,-6);
    \draw[red] (4.7,-8) -- (7.3,-4);
    
    \draw[red] (4.7,-6) -- (7.3,-8);
    \draw[red] (4.7,-6) -- (7.3,-6);
    \draw[red] (4.7,-6) -- (7.3,-4);

    \draw[red] (4.7,-4) -- (7.3,-8);
    \draw[red] (4.7,-4) -- (7.3,-6);
    \draw[red] (4.7,-4) -- (7.3,-4);

    \draw[green] (4.7,-2) -- (7.3,4);
    \draw[green] (4.7,-2) -- (7.3,6);
    \draw[green] (4.7,-2) -- (7.3,8);

    \draw[green] (4.7,2) -- (7.3,4);
    \draw[green] (4.7,2) -- (7.3,6);
    \draw[green] (4.7,2) -- (7.3,8);

    \draw[green] (4.7,0) -- (7.3,4);
    \draw[green] (4.7,0) -- (7.3,6);
    \draw[green] (4.7,0) -- (7.3,8);

    \draw[blue] (4.7,8) -- (7.3,4);
    \draw[blue] (4.7,8) -- (7.3,6);
    \draw[blue] (4.7,8) -- (7.3,8);

    \draw[blue] (4.7,6) -- (7.3,4);
    \draw[blue] (4.7,6) -- (7.3,6);
    \draw[blue] (4.7,6) -- (7.3,8);

    \draw[blue] (4.7,4) -- (7.3,4);
    \draw[blue] (4.7,4) -- (7.3,6);
    \draw[blue] (4.7,4) -- (7.3,8);

    \node at (0,0) {$x_2$};
    \node at (0,3) {$x_1$};
    \node at (0,-3) {$x_3$};

    \node at (8,0) {$y_2$};
    \node at (8,2) {$y_2$};
    \node at (8,-2) {$y_2$};

    \node at (8,4) {$y_1$};
    \node at (8,6) {$y_1$};
    \node at (8,8) {$y_1$};
    \node at (8,-4) {$y_3$};
    \node at (8,-6) {$y_3$};
    \node at (8,-8) {$y_3$};
    \draw[<->][black] (2,8.7) -- (2,3.2);
    \node at (0.7,6) {$\log_2\left(\frac{1}{\delta}\right)$};
\end{tikzpicture}};
\end{tikzpicture}
    \caption{The Figure shows representation of first two layers of a network in Theorem \ref{thm:main_2}: (\cite{Burkholz}). The figure on the left shows the target network, where as Figure on the right shows the large network. The colors indicate which part in the target is represented by which part of the source. For example, the red weight on the left is represented by the red weights on the right.}
    \label{fig:f4}
\end{figure}
\newpage


\end{document}